\documentclass{article}

\usepackage{arxiv}

\usepackage[utf8]{inputenc} 
\usepackage[T1]{fontenc}    
\usepackage{hyperref}       
\usepackage{url}            
\usepackage{booktabs}       
\usepackage{amsfonts}       
\usepackage{nicefrac}       
\usepackage{microtype}      
\usepackage{lipsum}

\usepackage{graphicx}   

\usepackage{amsmath}
\usepackage{subcaption}
\usepackage{multirow}
\usepackage{tcolorbox}
\usepackage{adjustbox}
\usepackage{hyperref}
\usepackage{cite}
\usepackage{amssymb}
\usepackage[inkscapelatex=false]{svg}
\usepackage{mathtools}
\usepackage{xcolor}
\usepackage{lipsum}
\usepackage{pifont}
\usepackage{threeparttablex}
\usepackage{float}
\usepackage{algorithm}
\usepackage[noend]{algpseudocode}
\algrenewcommand\algorithmicrequire{\textbf{Input:}}
\algrenewcommand\algorithmicensure{\textbf{Output:}}
\usepackage{stmaryrd}

\newcommand{\R}{{\mathbb{R}}}

\newcommand{\B}{{\mathbb B}}
\newcommand{\cen}{{\mathbf{c}}}
\newcommand{\rad}{{\mathbf{r}}}
\newcommand{\N}{{\mathbb{N}}}
\newcommand{\G}{{\mathbf{G}}}

\newcommand{\X}{{\mathbf{X}}}

\newcommand{\T}{{\mathbf{T}}}
\newcommand{\sbo}{{\mathsf{s}}}
\newcommand{\So}{{\mathbf{S}}}
\newcommand{\Obs}{{\mathcal{O}}}

\newcommand{\tx}{{t_{\times}}}

\newcommand{\cmark}{\ding{51}}
\newcommand{\xmark}{\ding{55}}

\DeclareMathAlphabet{\mathbbold}{U}{bbold}{m}{n}

\newtheorem{theorem}{Theorem}[section]

\newtheorem{assumption}{Assumption}

\newenvironment{proof}{Proof:}{\hfill$\square$}
\newtheorem{definition}[theorem]{Definition}
\newtheorem{lemma}[theorem]{Lemma}
\newtheorem{remark}[theorem]{Remark}

\title{Learning Spatiotemporal Tubes for Full Class of Signal Temporal Logic Tasks for Control of Unknown Systems under Input Constraints
\thanks{The work of Ratnangshu Das is supported by the Prime Minister’s Research Fellowship from the Ministry of Education, Government of India.}
}

\author{
 Ahan Basu \\
  Centre for Cyber-Physical Systems\\
  Indian Institute of Science, Bengaluru, India\\
  \texttt{ahanbasu@iisc.ac.in} \\
   \And
 Ratnangshu Das \\
   Centre for Cyber-Physical Systems\\
  Indian Institute of Science, Bengaluru, India\\
  \texttt{ratnangshud@iisc.ac.in} \\
  \And
  Soumyodipta Nath\\
   Centre for Cyber-Physical Systems\\
  Indian Institute of Science, Bengaluru, India\\
  \texttt{soumyodiptan@iisc.ac.in} \\
  \And
  Siyuan Liu \\
   Electrical Engineering Department, \\ Eindhoven Institute of Technology, Netherlands \\
  \texttt{s.liu5@tue.nl} \\
  \And
 Pushpak Jagtap \\
  Centre for Cyber-Physical Systems\\
  Indian Institute of Science, Bengaluru, India\\
  \texttt{pushpak@iisc.ac.in} \\
}

\begin{document}

\maketitle

\begin{abstract}
This paper presents a Spatiotemporal Tube (STT)-based control framework for general unknown nonlinear Euler–Lagrange (EL) systems subject to input constraints, with the objective of satisfying Signal Temporal Logic (STL) specifications, {where confinement of the system trajectory within the STT guarantees the satisfaction of the corresponding STL task.} For both single and multi-agent scenarios, the STT corresponding to each agent is modeled as a time-varying ball, whose center and radius are jointly parameterized using a physics-informed neural network (PINN). The robustness metric associated with the STL specification corresponding to the agents is incorporated into the training process as a loss function, enabling the learned tube to encode task-level temporal requirements. For a multi-agent scenario, we introduce an additional robustness metric corresponding to the global task, which, when satisfied, ensures the tubes do not collide with each other. To ensure that the system trajectory remains within the learned STT and thereby satisfies the local and global STL specifications, we propose a control strategy that explicitly accounts for input constraints. In particular, a closed-form control law is developed to keep the trajectory inside the tube while regulating the motion of the tube by enforcing bounds on its evolution depending on the input constraints of the system. The proposed approach has been validated over several case studies. 
\end{abstract}


\section{Introduction}

Modern control applications, such as safe and reliable operations of robots and unmanned vehicles, demand satisfying complex time-dependent task specifications instead of the traditional stabilization and trajectory tracking problem. Signal Temporal Logic (STL) \cite{maler2004monitoring} provides a formal language framework to describe these tasks, such as requiring a robot to transport materials between specified locations within a prescribed time horizon, and avoid time-varying hazardous regions, like areas undergoing scheduled maintenance. Moreover, STL provides a quantitative measure of task satisfaction through robustness metrics \cite{donze2010robust}. Essentially, STL has become an increasingly popular tool in areas such as deep learning \cite{liu2021recurrent}, reinforcement learning \cite{saxena2023funnel}, learning from demonstration \cite{puranic2021learning}, planning and control \cite{sun2022multi,juvvi2025signal}. Despite its wide applicability, control under STL specifications is challenging due to algorithmic and computational issues.

Mixed-Integer Programming (MIP) \cite{raman2015reactive} and gradient-based methods \cite{leung2023backpropagation} have been used to enforce STL constraints, but because they rely on optimization, they often do not scale well as the STL specifications become more complex or the system dynamics become more complicated. In recent years, Model Predictive Control has been deployed as a promising tool to satisfy STL specifications within a receding-horizon framework \cite{raman2014model}, but it also often suffers from computational complexity issues for complex tasks. Symbolic control provides formal correctness guarantees but is prone to the curse of dimensionality due to state-space abstraction \cite{tabuada2009verification}. 

An alternative approach to satisfy STL specifications is using the concept of Prescribed Performance Control (PPC), where performance funnels have been enforced on STL robustness metrics \cite{lindemann2021funnel, liu2022compositional}. Although computationally efficient, it can handle only a fragment of STL specifications. Control Barrier Function (CBF) \cite{ames2019control} enforces safety constraints with formal guarantees, and time-varying CBF has also been used to under-approximate the predicate function \cite{lindemann2018control}. However, it requires accurate system dynamics and is also limited to a fragment of STL specifications.   

The Spatiotemporal Tube (STT) framework, which provides a time-varying safe region in the state space, was introduced in \cite{das2025spatiotemporal} and subsequently extended to address the satisfaction of the full class of Signal Temporal Logic (STL) specifications \cite{STT_STL}. In these works, the tubes are generated by solving an optimization problem requiring a predefined template of the tubes, which restricts their expressivity and suffers from increased computational complexity for more complex specifications. Moreover, these tubes rely on unbounded control inputs and can not handle the input constraints, making it infeasible for real-world systems. 

In recent years, leveraging the universal approximation property of neural networks, several efforts have been made to synthesize safety controllers for STL specifications \cite{liu2021recurrent, liu2023learning}, which often struggle to enforce formal correctness guarantees, require additional computational burden, and can be applied only for some fragment of the STL specification. In addition, these neural network-based methods still rely on extensive state-space data, necessitating access to a black-box or simulator model for generating training samples, which is often difficult to obtain in real-world scenarios.

In this work, we introduce the Physics-informed Neural Spatiotemporal Tube (PINSTT) framework that constructs the STT directly for the full class of STL specifications, without the knowledge of system dynamics and state-space samples. We frame the STL robustness metric \cite{donze2010robust} as the loss function of the neural network and ensure a positive radius of the tube by forming another loss function. A key contribution of our work is the use of STT with spherical cross-sections, which significantly reduces computational complexity compared to its hyper-rectangle counterparts \cite{das2025spatiotemporal, STT_STL}. The training of the neural network is performed over the collocation points collected over the augmented time horizon. We verify the trained PINSTT over the continuous horizon on the fly using a Lipschitz-based validity condition. Leveraging the automatic differentiation property of the physics-informed neural network \cite{sharma2022accelerated}, we enforce the Lipschitz boundedness of the center and radius of the constructed tube. The Lipschitz bounds are chosen in such a way that the tube moves in conjunction with the input constraints of the system. Then, based on the velocity and torque bounds of the system, we derive feasibility conditions and subsequently provide an approximation-free, closed-form control law to keep the system within the tube, thereby ensuring satisfaction of the STL specification. Finally, the approach is validated over several case studies. 

\textit{Notations:}
The symbols $\N$, $ \R$, $\R^+$, and $\R_0^+$ denote the set of natural, real, positive real, and nonnegative real numbers. 
The Euclidean norm of a vector is denoted by $\lVert \cdot \rVert$.
For $a, b \in \mathbb{N}$ with $a \leq b$, the closed interval in $\mathbb{N}$ is denoted as $[a; b]$. 
The set difference between two sets $\mathcal{A}$ and $\mathcal{B}$ is defined as $\mathcal{A}\setminus \mathcal{B}$.
A ball centered at $\cen \in \mathbb{R}^n$ with radius $\rad \in \mathbb{R}^+$ is defined as $\B(\cen, \rad) := \{ x \in \mathbb{R}^n \mid \|x - \cen\| \leq \rad \}$. For all $x, y \in \R^n$, the vector inequality $x \preceq y$ represents $x_i \leq  y_i$ for all $i \in [1;n]$.
We define the sup norm $\|\cdot\|_{[a,b]}$ of a signal $s:\R_0^+ \rightarrow \R^n$ over time interval $[a,b]$ as $\|s\|_{[a,b]}:=\min_{t\in [a,b]}\|s(t)\|$.
The symbol $x \odot y$ represents the Hadamard product (elementwise multiplication) of the two vectors of $x, y \in \R^n$. 
We use $\mathbf{I}_n$ and $\mathbf{0}^{n \times m}$ to denote the identity matrix of size $n \times n$ and the zero matrix of size $n \times m$, respectively. 
$x \uparrow (\downarrow) c$ implies $x$ approaches $c$ from the left (right) side.
\section{Preliminaries and Problem Formulation}
\label{sec:prelim}

\subsection{System Definition}
Consider an Euler-Lagrange (EL) system $\mathcal{S}$ given as
\begin{align}
    \mathcal{S}: M(x)\Ddot{x} + V(x,\dot{x}) + G(x) = \tau(t) + w(t), \label{eqn:sysdyn}
\end{align}
where $x(t) = [x_1(t), \ldots, x_n(t)]^\top \in \X \subset \mathbb{R}^n$ is the system configuration, $\tau(t) \in \mathbb{R}^n$ is the control input, and $w(t) \in \mathbb{R}^n$ is an unknown disturbance. The terms $M(x) \in \mathbb{R}^{n \times n}$, $V(x,\dot{x}) \in \mathbb{R}^n$, and $G(x) \in \mathbb{R}^n$ are the inertia, Coriolis/centrifugal, and gravity components. For brevity, we omit the functional dependence and denote the terms by $M$, $V$, $G$, and $w$.


For $\mathcal{S}$, the control input is bounded, i.e., $|\tau(t)| \preceq \overline{\tau}, \quad \forall t \in \R_0^+.$ This induces implicit bounds on the system dynamics \cite[Chapter 2]{ELbook}, \cite[Chapter 7]{ELbook2}, \cite{ELbounds}. Although the disturbance $w$ and the system parameters $M$, $V$, and $G$ are unknown, their boundedness can be used for control design. To formalize this, we introduce the following assumptions.

\begin{assumption}\label{assum_d}
    The external disturbance $w$ is bounded between $-\overline{w} \preceq w(t) \preceq \overline{w},$ for all $t \in \R_0^+$, where $\overline{w} \in \R^n$ is a known bound.
\end{assumption}

\begin{assumption}\label{assum_tau}
    Given the control bound $\overline{\tau}$, there exists a positive constant $\underline{m} \in \R$, such that $\underline{m}\overline{\tau} \preceq M^{-1}\overline{\tau}$.
\end{assumption}

\begin{assumption}\label{assum_V}
The Coriolis and centrifugal terms $V$ and the gravity vector $G$ satisfy $\underline{V}_M \preceq V_M \preceq \overline{V}_M$, where $V_M := -M^{-1}(V+G)$ and $\underline{V}_M, \overline{V}_M \in \R^n$.
\end{assumption}

\begin{assumption}\label{assum_Md}
    The inverse of the mass matrix scales the disturbance as $-\|M^{-1}\| \overline{w} \preceq M^{-1}d \preceq \|M^{-1}\| \overline{w}$. This implies, there exists $\underline{m}_i \in \R^+$, such that $-\underline{m}_i \overline{w} \preceq M^{-1}d \preceq \underline{m}_i \overline{w}$.
\end{assumption}

The Assumptions \ref{assum_d}-\ref{assum_Md} provide the bounds on various system parameters, which will be used to derive feasibility conditions ensuring that both the STL and input constraints are satisfied.

\subsection{Signal Temporal Logic}
Signal Temporal Logic (STL) \cite{maler2004monitoring} is a formal language used to specify the spatial, temporal, and logical properties of continuous-time signals. The set of STL formulae can be recursively expressed using predicates $\mathsf{p}$. Consider the predicate function $h: \R^n \rightarrow \R$ to be Lipschitz continuous, and $h(s) \geq 0 \Rightarrow \mathsf{p}:=\textsf{true}$, else $\mathsf{p}:=\textsf{false}$. An STL formula $\phi$ is recursively defined using predicates, Boolean logic, and temporal operators:
\begin{align}\label{eq:STL}
    \phi:= \textsf{true} \ | \ \textsf{p} \ | \ \neg \phi \ | \phi_1 \land \phi_2 \ | \phi_1 \mathcal{U}_{[a,b]}\phi_2, 
\end{align}
where $\phi_1, \phi_2$ are STL formulae. Boolean operators for negation and conjunction are denoted by $\neg$ and $\land$ respectively. $\mathcal{U}$ represents the temporal until operator in the time interval $[a,b]$ with $a,b \in \R_0^+$ such that $a \leq b$. The satisfaction relation $(x,t) \models \phi$ denotes a signal $x : \R_{\geq 0} \rightarrow \R^n$, a possible state of \eqref{eqn:sysdyn}, satisfies an STL formula $\phi$ at time $t$. The STL semantics \cite{maler2004monitoring} for a signal $x$ is recursively given by:
\begin{align*}
    (x,t) \models \textsf{p} \iff& h(x(t)) \geq 0, \\
    (x,t) \models \neg \phi \iff& \neg((x,t) \models \textsf{p}), \\
    (x,t) \models \phi_1 \land \phi_2 \iff& (x,t) \models \phi_1 \land (x,t) \models \phi_2, \\
    (x,t) \models \phi_1 \mathcal{U}_{[a,b]} \phi_2 \iff& \exists \ t_1 \in [t+a, t+b], (x,t_1) \models \phi_1 \land \forall t_2 \in [t+a, t_1], (x, t_2) \models \phi_2.
\end{align*}
The disjunction, eventually and always operator can be derived as $\phi_1 \lor \phi_2 = \neg(\neg \phi_1 \land \neg \phi_2), \Diamond_{[a,b]}\phi = \textsf{true} \ \mathcal{U}_{[a,b]}\phi$ and $\square_{[a,b]}\phi = \neg \Diamond_{[a,b]}\neg \phi$. A signal $x$ satisfies an STL formula $\phi$, denoted by $x \models \phi$ if and only if $(x, 0) \models \phi$. The STL robustness metric $\rho^\phi(x,t)$ \cite{donze2010robust} quantifies the degree of satisfaction: $\rho^\phi(x,t) > 0$ implies $(x,t) \models \phi$ and its magnitude reflects the strength of satisfaction or violation. The robustness semantics are defined recursively as:
\begin{subequations}\label{eq:stl_sem}
\begin{align}
    &\rho^{\phi}(x,t) = h(x(t)), \\
    &\rho^{\neg \phi}(x,t) = -\rho^{\phi}(x,t), \\
    &\rho^{\phi_1 \land \phi_2}(x,t) = \min \big(\rho^{\phi_1}(x,t) \land \rho^{\phi_2}(x,t)\big),\\
    & \rho^{\Diamond_{[a,b]}\phi}(x,t) = \max_{t_1 \in [t+a, t+b]}\rho^\phi(x, t_1),\\
    & \rho^{\square_{[a,b]}\phi}(x,t) = \min_{t_1 \in [t+a, t+b]}\rho^\phi(x, t_1),\\
    & \rho^{\phi_1 \ \mathcal{U}_{[a,b]}\phi_2} (x,t) = \max_{t_1 \in [t+a, t+b]}\min \big(\rho^\phi_1(x, t_1), \min_{t_2 \in [t+a, t_1]}\rho^{\phi_2}(x, t_2) \big).
\end{align} 
\end{subequations}
For simplicity, robustness at $t=0$ is written as $\rho^\phi(x)$. We consider a finite time horizon $[0, t_f]$ over which the specification $\phi$ to be realized, i.e., $x:[0, t_f] \rightarrow \R^n $ such that $ x \models \phi$.

\subsection{Spatiotemporal Tubes}
To satisfy the STL specification, we leverage Spatiotemporal Tubes \cite{das2025spatiotemporal, STT_STL}, defined next.
\begin{definition}[STT for STL] \label{def:stt}
    For an STL task $\phi$ as in \eqref{eq:STL} defined over the time interval $[0, t_f]$, the time-varying ball $\Gamma\big(t \big) = \B(\cen(t), \rad(t)):=\{x \in \R^n | \quad \lVert x - \cen(t) \rVert \leq \rad(t)\}$, where $\rad: \R_0^+ \rightarrow \R^+$ and $\cen(t):=[\cen_1(t), \ldots, \cen_n(t)]^\top$ (with $\cen_i: \R_0^+ \rightarrow \R, \ i \in [1;n]$) are continuously differentiable functions, is called a Spatiotemporal Tube (STT), if the following holds:
    \begin{align}\label{eq:stt}
        \rho^{\phi}(s) \!> \!0, \! \forall s \!:\![0,t_f] \!\rightarrow \!\R^n, s.t., s(\tau)\! \in \!\Gamma(\tau), \! \forall \tau \!\in \![0, t_f].
    \end{align}
\end{definition}
\begin{remark}
    If one designs a control law that restricts the state of \eqref{eqn:sysdyn} within the STT, i.e., $x(t) \in \Gamma \big(t \big), \forall t \in [0, t_f]$, then one can ensure that the system satisfies the STL task.
\end{remark}

\subsection{Problem Statement}
Given an STL specification $\phi$ over time $t \in [0,t_f]$ as in equation \eqref{eq:STL}, the objective is to synthesize the STTs, as introduced in Definition~\ref{def:stt}. Then, for the unknown EL system in \eqref{eqn:sysdyn}, the goal is to design a closed-form, approximation-free control law $\tau(t)$ satisfying the input constraints $-\bar{\tau} \preceq \tau(t) \preceq \bar{\tau},$ such that the system configuration $x(t)$ remains constrained within the synthesized STTs for all time $t \in [0,t_f]$. By guaranteeing the forward invariance of the STTs, the resulting trajectory satisfies the STL specification, i.e., $x \models \phi$.

The central idea of the proposed framework is to transform STL satisfaction into a geometric constraint satisfaction problem. Rather than reasoning directly over the STL specification, we aim to represent a set of admissible trajectories through a time-varying tube and subsequently design a controller that ensures the system remains within this tube at all times.


\section{Neural STT for STL specifications} \label{sec:neural_STT}

The development of the proposed framework begins with the synthesis of STTs associated with the given STL specification. To this end, we first establish sufficient conditions that guarantee a time-varying ball encapsulates a subset of the trajectories satisfying the specification. These conditions form the basis for the neural STT construction presented subsequently.

\begin{lemma}\label{lem:ROP}
    The time-varying ball $\Gamma(t)$, characterized by its center $\cen(t)$ and radius $\rad(t)$, is considered a valid STT as in Definition \ref{def:stt} such that it encapsulates signals satisfying the STL specification as in Equation \eqref{eq:STL} if the following conditions are satisfied with $\eta \leq 0$:
    \begin{subequations}\label{eq:ROP}
    \begin{align}
        & \forall t \in [0,t_f]: -\rad(t) + \rad_d \leq \eta, \label{eq:radius} \\
        & \forall x: [0, t_f] \rightarrow \R^n \ \text{s.t.,}\ x(\tau) = \cen(\tau) + \lambda \rad(\tau) \sbo(\theta), \forall (\theta, \lambda, \tau) \in [0,2\pi]^{n-1} \times [0,1] \times [0,t_f] :  -\rho^\phi(x) < \eta,\label{eq:robust_ROP}
    \end{align}
    \end{subequations}
    where $\sbo(\theta) = [\cos\theta_1, \sin\theta_1\cos\theta_2, \ldots,
    \sin\theta_1\ldots\sin\theta_{n-2}\cos\theta_{n-1},  \sin\theta_1\sin\theta_2\ldots\sin\theta_{n-2}\sin\theta_{n-1}]^\top$ parametrizes unit sphere and $\rad_d \in \R^+$ is a user-defined lower bound on tube radius.
\end{lemma}
\begin{proof}
    The condition \eqref{eq:radius} being satisfied with $\eta \leq 0$ implies that the tube radius is always positive, ensuring a feasible safe set of non-zero volume for all time. The condition \eqref{eq:robust_ROP} with $\eta \leq 0$ ensures that the STL robustness remains strictly positive inside the tube. $\sbo(\theta)$ covers all directions of the unit sphere while $\lambda \in [0,1]$ scales radially outward from the center of the tube to its boundary. Altogether, they span the complete ball $\B(\cen(t), \rad(t))$ ensuring that any trajectory confined within the STT satisfies the STL specification.
\end{proof}

\subsection{Sample Collection and Training Procedure}
It is evident that \eqref{eq:ROP} has infinitely many constraints due to the continuous space-time domain, rendering the problem intractable in its original form. To obtain a finite and computationally tractable solution, we sample $N$ points from the augmented space-time horizon $\mathcal{W} = [0,2\pi]^{n-1} \times [0,1] \times [0, t_f]$, denoted by $\mathsf{w}_r = (\theta_r, \lambda_r, t_r)$, $r\in [1;N]$. We consider a ball $\B(\mathsf{w}_r, \varepsilon)$ around each sample $\mathsf{w}_r$ with radius $\varepsilon$, such that there exists a $\mathsf{w}_r$ for all $(\theta, \lambda, t) \in \mathcal{W}$ satisfying 
\begin{align}\label{eq:sample}
    \lVert (\theta, \lambda, t) - \mathsf{w}_r \rVert \leq \varepsilon. 
\end{align}
This ensures that the union of such balls forms a superset of the augmented space-time domain: $\cup_{r=1}^M \B(\mathsf{w}_r, \varepsilon) \supset \mathcal{W}$.

We now seek a functional representation of the STT $\Gamma(t)$ to satisfy these sampled constraints,
\begin{subequations}\label{eq:SOP}
\begin{align}
    & \forall t_r \in [0,t_f], r \in [1;M]: -\rad(t_r) + \rad_d \leq \eta, \label{eq:SOP_radius} \\
    & \forall x: [0, t_f] \rightarrow \R^n \ \text{s.t.,}\ x(\tau_r) = \cen(\tau_r) + \lambda \rad(\tau_r) \sbo(\theta_r), \forall (\theta_r, \lambda_r, t_r) \in \mathcal{W}, r \in [1;M] :  -\rho^\phi(x) < \eta.\label{eq:robust_SOP}
\end{align}
\end{subequations}
Since the functional form of the tube center and radius is not known a priori, we model them using a neural network. Although a standard artificial neural network (ANN) can approximate the time-varying ball, training only at discrete samples does not formally guarantee smoothness or inter-sample consistency. This is undesirable because the STT center $\cen(t)$ and radius $\rad(t)$ are required to evolve smoothly over time. To overcome the limitation, we use a physics-informed neural network (PINN) framework, namely Physics-informed Neural Spatiotemporal Tube (PINSTT), 
\begin{align}\label{eqn:pinstt}
    \Gamma (t;\Upsilon) = \B(\cen(t;\Upsilon), \rad(t;\Upsilon)),
\end{align} 
with center $\cen(t;\Upsilon) = [\cen_1(t;\Upsilon), \ldots, \cen_n(t;\Upsilon)]^\top \in \R^n$, radius $\rad(t;\Upsilon) \in \R$, and the trainable parameters $\Upsilon$. 

During training of the PINN, the finite set of constraints in \eqref{eq:SOP} is encoded as physics-based penalties. Specifically, the loss function that is to be minimized is given by $\mathsf{L}(\Upsilon)$, which only consists of the physics-based sub-losses formulated from the constraints. The sub-loss terms are defined as:
\begin{subequations}\label{loss:phys_1}
\begin{align}
    \mathsf{L}_1(\Upsilon) &= \sum_{r=1}^M \text{ReLU} \big(-\rad(t_r; \Upsilon) + \rad_d - \hat{\eta} \big), \\
    \mathsf{L}_2(\Upsilon) &= \sum_{r=1}^M \text{ReLU} \big(-\rho^\phi(y(\mathsf{w}_r)) + \delta - \hat{\eta} \big),
\end{align}
\end{subequations}
where $\text{ReLU}(x):= \max(0,x)$, and the choice of $\hat{\eta}$ is described in the Algorithm~\ref{algo:NN_training} while $\delta \in \R^+$ is a small positive quantity to ensure the strict inequality corresponding to condition \eqref{eq:robust_SOP}. 

To ensure that the tube varies smoothly over time, we require $\cen(t; \Upsilon)$ and $\rad(t; \Upsilon)$ to be Lipschitz continuous, with user-defined bounds $\mathcal{L}_{\cen}$ and $\mathcal{L}_{\rad}$, respectively. {The detailed procedure of choosing the parameters $\mathcal{L}_c$ and $\mathcal{L}_r$ will be discussed in Section \ref{sec:control}.} This assumption is mild in practice, as Lipschitz continuity prevents abrupt temporal changes in the tube and, in turn, leads to smoother control inputs. In addition, neural networks equipped with slope-restricted activation functions are inherently Lipschitz continuous. To enforce these bounds, we use the automatic differentiation property of the PINNs \cite{raissi2019physics} and ensure the derivative of the output is bounded by the user-defined bounds. Accordingly, we introduce two additional sub-loss functions, defined as:
\begin{subequations}\label{loss:phys_2}
\begin{align}
    \mathsf{L}_3(\Upsilon) &= \sum_{r=1}^M \text{ReLU}(\| \dot{\cen}(t_r; \Upsilon) \| - \mathcal{L}_c), \\
    \mathsf{L}_4(\Upsilon) &= \sum_{r=1}^M \text{ReLU} (| \dot{\rad}(t_r; \Upsilon) | - \mathcal{L}_r).
\end{align}
\end{subequations}
Therefore, the combined loss is given by 
\begin{align}\label{eq:loss_phys}
    \mathsf{L}(\Upsilon) = \sum_{i=1}^4 w_i \mathsf{L}_i(\Upsilon), w_i\in \R_0^+, i \in [1;4],
\end{align}
where $w_i$ are the weights corresponding to the sub-loss terms. Algorithm~\ref{algo:NN_training} summarizes the training procedure.

\subsection{Formal Verification of the Network}
We now verify that the trained PINSTT captures the STL specification over the entire continuous space-time horizon. From \cite{akella2023lipschitz}, we know that the STL robustness function $\rho^\phi(x)$ is Lipschitz continuous with respect to signal {as all predicate functions are Lipschitz continuous and the signal evolves within the bounded subset $\X$ of the state space $\R^n$}, which means there exists $\mathcal{L}_\rho \in \R^+$ such that for any $x_1,x_2 : [0, t_f] \rightarrow \X \subset \R^n$,
$$\lVert \rho^\phi(x_1) - \rho^\phi(x_2) \rVert \leq \mathcal{L}_\rho \lVert x_1 - x_2 \rVert_{[0,t_f]}.$$ 
{Clearly, in our setting, the signal $x(t) = \cen(t;\Upsilon) + \lambda \rad(t; \Upsilon)\mathsf{s}(\theta)$ is confined within the tube, which is confined within the state-space, it ensures the robustness metric $\rho^\phi(x)$ admits a finite Lipschitz constant over the domain.}
\begin{lemma}[\cite{das2025control}] \label{lem:Lipschitz_robust}
    $\mu(\theta, \lambda):= -\rho^\phi(x)$ with $x(\tau) = \cen(\tau; \Upsilon) + \lambda \rad(\tau; \Upsilon) \sbo(\theta)$ for all $\tau \in [0,t_f]$, is Lipschitz continuous with the Lipschitz constant $\mathcal{L}_{\mu}:= \mathcal{L}_\rho \Bar{\rad} \sqrt{\mathcal{L}_s^2+1} \in \R^+$. Therefore, for all $(\theta, \lambda), (\hat{\theta}, \hat{\lambda}) \in [0,2\pi]^{n-1} \times [0,1]: \lVert \mu(\theta, \lambda) - \mu(\hat{\theta}, \hat{\lambda}) \rVert \leq \mathcal{L}_\mu \lVert (\theta, \lambda) - (\hat{\theta}, \hat{\lambda}) \rVert$.
\end{lemma}
\begin{proof}
    We first get the Lipschitz constants with respect to the individual parameters $\theta$ and $\lambda$, and then combine the results to obtain the overall Lipschitz constant.

    From \cite{akella2023lipschitz}, we know that there exists $\mathcal{L}_\rho \in \R^+$ such that $| \rho^\phi(x_1, \ldots, x_k, \ldots, x_n) -  \rho^\phi(\hat{x}_1, \ldots, \hat{x}_k, \ldots, \hat{x}_n)| \leq \mathcal{L}_\rho |x_k -\hat{x}_k|$. $\lVert\mathsf{s}(\theta)\rVert \leq 1$ for all $\theta$ with Lipschitz bound $\mathcal{L}_s$ and $\Bar{\rad} = \max_{t \in [0,t_f]}\rad(t_r; \Upsilon)$ is finite.
    \begin{itemize}
        \item[(i)] For some fixed $\lambda \in [0,1]$, for any $\theta, \hat{\theta} \in [0,2\pi]^{n-1}$, we have $\lvert \mu(\theta,\lambda) - \mu(\hat{\theta},\lambda) \rvert \leq \mathcal{L}_\rho \Bar{\rad} \mathcal{L}_s \|\theta - \hat{\theta}\|$.

        \item[(ii)] For some fixed $\theta \in [0,2\pi]^{n-1}$, for any $\lambda, \hat{\lambda} \in [0,1]^n$, we have $\lvert \mu(\theta,\lambda) - \mu(\theta, \hat{\lambda}) \rvert \leq \mathcal{L}_\rho \Bar{\rad} |\lambda - \hat{\lambda}|$. 
    \end{itemize}
    Now combining both of them, for any $(\theta, \lambda), (\hat{\theta}, \hat{\lambda}) \in [0,2\pi]^{n-1} \times [0,1]$, we can have:
    \begin{align*}
        &| \mu(\theta, \lambda) \!-\! \mu(\hat{\theta}, \hat{\lambda})| = \lvert \mu(\theta,\lambda) \!-\! \mu(\hat{\theta},\lambda) \rvert \!+\! \lvert \mu(\hat{\theta},\lambda) \!-\! \mu(\hat{\theta},\hat{\lambda}) \rvert \leq \mathcal{L}_\rho \Bar{\rad} \mathcal{L}_s\lVert \theta \!-\! \hat{\theta}\rVert + \mathcal{L}_\rho \Bar{\rad} |\lambda \!-\! \hat{\lambda}| \leq \mathcal{L}_\mu \lVert (\theta, \lambda) \!-\! (\hat{\theta}, \hat{\lambda}) \rVert,
    \end{align*}
    where $\mathcal{L}_\mu:= \mathcal{L}_\rho \Bar{\rad} \sqrt{\mathcal{L}_s^2+1}$, which is obtained using Cauchy-Schwarz inequality. This completes the proof.
\end{proof}

We now present the main theorem of the section that the PINSTT obtained after training using the sampled points is formally verified to encapsulate the STL specification over the continuous space-time horizon. 

\begin{theorem}\label{th:constr}
    The PINSTT $\Gamma(t; \Upsilon)$ in \eqref{eqn:pinstt} obtained upon training as per Algorithm~\ref{algo:NN_training} using the sampled points as in \eqref{eq:sample}, is guaranteed to satisfy the conditions~\eqref{eq:ROP} provided that the constraints in \eqref{eq:SOP} hold and 
    $$\hat{\eta} + \mathcal{L}\varepsilon \leq 0, \ \mathcal{L}=\max\{\mathcal{L}_r, \sqrt{\mathcal{L}_\mu^2 + \mathcal{L}_\rho^2(\mathcal{L}_c^2+ \mathcal{L}_r^2)} \},$$
    where $\mathcal{L}_c, \mathcal{L}_r$ are the Lipschitz constants of the tube center and radius, while $\mathcal{L}_{\mu}$ and $\mathcal{L}_{\rho}$ are defined in Lemma \ref{lem:Lipschitz_robust}.
\end{theorem}
\begin{proof}  
    From \eqref{eq:sample}, for every $t \in [0,t_f]$, there exists a sampled point $t_r$ such that $|t - t_r| \leq \varepsilon.$ Therefore, for all $t \in [0,t_f], \ -\rad(t; \Upsilon) + \rad_d \leq -\rad(t; \Upsilon) + \rad(t_r; \Upsilon) - \rad(t_r; \Upsilon) + \rad_d \leq \hat{\eta} + \mathcal{L}_r \varepsilon \leq \hat{\eta} + \mathcal{L} \varepsilon \leq 0$.

    Similarly, from \eqref{eq:sample}, for every $\mathsf{w} \!=\! (\theta,\lambda,\tau) \in \mathcal{W}$, there exists $\mathsf{w}_r \!=\! (\theta_r,\lambda_r,t_r)$, s.t., $\|(\theta,\lambda,t) \!-\! (\theta_r,\lambda_r,t_r)\| \!\leq\! \varepsilon$. Define $y, \hat{y}$ as $y(\mathsf{w}) = \cen(\tau; \Upsilon) + \lambda \rad(t; \Upsilon)\mathsf{s}(\theta), \hat{y}(\mathsf{w}_r) = \cen(t_r; \Upsilon) + \lambda_r \rad(t_r; \Upsilon)\mathsf{s}(\theta_r)$. Let, $\bar{y}:[0,t_f] \rightarrow \R^n$ be the ZOH reconstruction of $\hat{y}$ over the sampled points. Then,
    \begin{align*}
        - \rho^\phi(y) &= - \rho^\phi(y) - \rho^\phi(\bar{y}) + \rho^\phi(\bar{y}) - \rho^\phi(\hat{y}) + \rho^\phi(\hat{y}) \\ 
        & \leq \| \mu(\theta, \lambda) - \mu(\theta_r, \lambda_r) \| + \hat{\eta}  + \max_{t \in [0,t_f], |t - t_r| \leq \varepsilon} \mathcal{L}_\rho \|\bar{y}_r(t) - {y}_r(t_r)\| \\
        & \leq \hat{\eta} + \mathcal{L}_\mu \lVert (\theta, \lambda) - (\theta_r, \lambda_r) \rVert + \mathcal{L}_\rho(\mathcal{L}_c + \mathcal{L}_r)|t -t_r| \
        \leq \hat{\eta} + \sqrt{\mathcal{L}_\mu^2 + \mathcal{L}_\rho^2(\mathcal{L}_c + \mathcal{L}_r)^2}\varepsilon \leq \hat{\eta} + \mathcal{L}\varepsilon \leq 0.
    \end{align*}
    Therefore, $\rho^\phi(y) > 0$ for all $y:[0,t_f] \rightarrow \R^n$, such that $y(t) \in \B(\cen(t; \Upsilon), \rad(t; \Upsilon))$ for all $t \in [0, t_f]$. This satisfies the conditions of Definition~\ref{def:stt}, completing the proof.
\end{proof}

\begin{algorithm}
\caption{Training of the PINSTT}
\label{algo:NN_training}
\begin{algorithmic}[1]
    \Require Space-time collocation data $\mathcal{W}:=\{w_r\}_{r=1}^M$, STL specifications, batch size, learning rate, maximum number of epochs. 
    \Ensure $\Gamma (t; \Upsilon)$
    \State {Initialize PINN, trainable parameter $\Upsilon$, hyperparameters $\mathcal{L}_c, \mathcal{L}_r$. Compute the Lipschitz constant $\mathcal{L}$. Initialize $\hat{\eta}$ such that $\hat{\eta} = - \mathcal{L}\varepsilon$.}
    \For{$i\leq Epochs$ (Training starts here)}
        \State Create batches of training data from $\mathcal{W}$.
        \State Find batch losses using \eqref{eq:loss_phys}.
        \State Update the NN parameters $\Upsilon$ using Adam/SGD optimizer \cite{ruder2016} to minimize the loss.
        \If{$\mathsf{L}(\Upsilon) \approx 0$}
            \State \textbf{break}
        \EndIf
        \If{$\mathsf{L}(\Upsilon)$ not close to zero and becomes static} 
            \State display: `Specification can not be achieved with maximum control'.
        \EndIf
    \EndFor
    \State \textbf{return} $\Gamma (t; \Upsilon)$
\end{algorithmic}
\end{algorithm}

\section{Controller Synthesis}\label{sec:control}
The controller is developed using a two-step procedure inspired by a backstepping-like design. First, a reference velocity is designed to guide the system configuration within the tubes. Then, an acceleration-level controller is constructed to track this reference while respecting input constraints.
For convenience, the EL system~\eqref{eqn:sysdyn} is rewritten as
\begin{subequations}
    \begin{align}
    \dot{x} &= v, \label{eqn:sysDyn_vel}   \\
    \dot{v} &= V_M(x,v) + M(x)^{-1}\tau + M(x)^{-1}d, \label{eqn:sysDyn_acc} 
    V_M(x,v) = -M(x)^{-1}(V(x,v)+G(x)). 
    \end{align}
\end{subequations}

\subsubsection*{Stage I: Velocity-level Control}

The goal of this stage is to design a reference velocity $v_r(t)$ that ensures the system configuration remains within the prescribed STT.

To enforce the system remains within the ball-shaped STTs $\Gamma(t)$ for all time $t \in [0;t_f]$, the reference velocity vector $v_r(x,t)$ is considered as
\begin{equation}\label{eqn:bdvelcon_ball}
    v_r(t) = -\frac{x(t) - \cen(t)}{\|x(t) - \cen(t)\|} \Psi(e_x) \overline{v}, \quad \overline{v} \in \R^+
\end{equation}
where
$e_x(x,t) := (\|x(t) - \cen(t)\|) / \rad(t)$ is the normalized position error.

\begin{remark}
    The map $\Psi:\R^n \rightarrow [-1,1]^n$ is a bounded transformation function, defined in \cite[Section~3.3]{das2026prescribedperformancecontrolunknown}, that ensures the control remains within admissible limits. 
\end{remark}

\subsubsection*{Stage II: Acceleration-level Control}
In this stage, the objective is to track the reference velocity $v_{r}(t)$ from Stage I, while ensuring that the control input remains bounded.
To achieve this, the velocity tracking error $e_v(t) = v(t) - v_{r}(t)$ is constrained within exponentially decaying funnel constraints $\gamma_v: \R_0^+ \rightarrow \R^n$
\begin{equation}\label{eqn:rhov}
    \gamma_v(t) = e^{-\mu_v t}(p_v-q_v) + q_v,
\end{equation}
as
\begin{equation}\label{eqn:fun2}
    -\gamma_v(t) \prec e_v(t) \prec \gamma_v(t).    
\end{equation}
Here $p_v \in \R^n$ is the initial funnel width, with $|e_v(0)| \preceq p_v$, $q_v \in \R^n$ is the steady state limit, with $0^{n \times 1} \prec q_v \prec p_v$, and $\mu_v \succeq \textbf{0}^{n \times n}$ determines the funnel decay rate.

This leads to the acceleration-level control input $\tau(t)$
\begin{equation}\label{eqn:bdcontrol}
    \tau(t) = -\text{diag}(\Psi(\varepsilon_v))\overline{\tau},
\end{equation}
where $ \varepsilon_v(t) = \text{diag}(\gamma_{v})^{-1}e_v(t)$ is the normalized velocity error and $\Psi$ is the same bounded transformation function, defined in \cite{das2026prescribedperformancecontrolunknown}, and $\overline{\tau} \in \R^n$ is the maximum permissible torque.

We now derive feasibility conditions under which the system remains within the synthesized STT under the available actuator limits.

\subsubsection*{Feasibility Condition}
We now derive feasibility conditions under which the velocity tracking error remains within the prescribed funnel bounds while ensuring that the synthesized STT is trackable under the available actuator limits. 
Since the tube center and radius evolve over time with rates bounded by $L_c$ and $L_r$, the system must possess sufficient velocity and actuation authority to follow this evolution. Therefore, the allowable values of $L_c$ and $L_r$ are fundamentally coupled with the velocity and input constraints, giving rise to the following feasibility conditions.


\textit{Stage I:}
From Section \ref{sec:neural_STT}, we have already ensured that $\|\dot{\cen}(t)\|$ and $|\dot{\rad}(t)|$ is upper bounded by $\mathcal{L}_c$ and $\mathcal{L}_r$. We define $\overline{\Gamma}:=\mathcal{L}_c + \mathcal{L}_r$ as the upper bound on the rate of variation of the tube over time, which dictates the maximum permissible velocity $\overline{v}$ given by
\begin{equation}\label{eqn:feas1}
    \overline{v} \geq \overline{\Gamma} + \|p_v\|,
\end{equation}
where $p_v \in \R^+$ is the starting position of the upper curve of the funnel. This guarantees that the system possesses sufficient actuation capacity to track the evolving tube boundaries $\overline{\Gamma}$ while rejecting the worst-case perturbation ($|e_v| \prec p_v$). It is evident that using the maximum permissible velocity $\bar{v}$, one can choose $\mathcal{L}_c$ and $\mathcal{L}_r$ such that \eqref{eqn:feas1} holds.

\textit{Stage II:}
Given the funnel constraint
$$\gamma_v(t) = e^{-\mu_vt}(p_v-q_v)+q_v,$$
and system dynamics in \eqref{eqn:sysdyn} with Assumptions \ref{assum_d}-\ref{assum_Md}, the maximum permissible torque $\overline{\tau}$ should adhere to the following constraint:
\begin{equation}\label{eqn:feas2}
    \overline{\tau} \succeq \frac{1}{\underline{m}} \left( \max(-\underline{V_M}, \overline{V_M})  +  \underline{m}_i\overline{w}  +  \mu_v(p_v-q_v)  +  \overline{a}_r \right),
\end{equation}
where given the bounded transformation function, in \cite{das2026prescribedperformancecontrolunknown}, we select an upper bound $\overline{a}_r \in \R^n$, such that $|\dot{v}_r| \preceq \overline{a}_r$. 

The following theorem formally summarizes the input-constrained approximation-free feedback controller.

\begin{theorem}\label{thm:bdcontrol}
Given the STL specification $\phi$ in \eqref{eq:STL}, let $\Gamma(t)$ denote the corresponding STT defined in \eqref{eqn:pinstt}.
For the EL system $\mathcal{S}$ in~\eqref{eqn:sysdyn} satisfying Assumptions~\ref{assum_d}--\ref{assum_Md}, suppose that the initial conditions satisfy
$x(0) \in \Gamma(0), \qquad |e_v(0)| \prec p_v,$
and that the feasibility conditions~\eqref{eqn:feas1} and~\eqref{eqn:feas2} hold.

Then, the closed-form control law $\tau(t)$ defined in Equation~\eqref{eqn:bdcontrol} guarantees that the system remains within the spatiotemporal tube, i.e., 
$$x(t) \in \Gamma(t), \ \forall t \in \R_0^+$$ 
Furthermore, the control input remains within the prescribed bound,
$$|\tau(t)| \preceq \overline{\tau}, \ \forall t \in \R_0^+.$$
Consequently, the closed-loop trajectory satisfies the STL specification, i.e., $x\models\phi$, under input constraints.
\end{theorem}
\begin{proof}
The proof is divided into two stages. For the first stage, we show that the reference velocity vector $v_r(t)$ in Equation \eqref{eqn:bdvelcon_ball} enforces the system state $x(t)$ to be confined within the STT $\Gamma(t)$, $\forall t \in \R_0^+$. In stage II, we prove that the control law $\tau(t)$ in \eqref{eqn:bdcontrol} keeps the velocity tracking error $e_v(t)$ within the funnel $[-\gamma_v(t), \gamma_v(t)]$ for all $t \in \R_0^+$ \eqref{eqn:fun2}. In both stages, we will show the result using the contradiction method.

\noindent\textbf{Stage 1:}
Let the first time instance when the system state $x(t)$ exits the STT $\Gamma(t)$ due to the applied velocity input $v_r(t)$ as of \eqref{eqn:bdvelcon_ball} is denoted by $\tx$.
Then,
\begin{gather}\label{Eq:inqe_tx_ball}
    \|x(t) - \cen(t)\| < \rad(t), \text{ for all } t \in [0, \tx).
\end{gather}
As the $x(t)$ approaches the STT boundary, $\|x(t) - \cen(t)\| \rightarrow \rad(t)$. For crossing the boundary, we have the following implications:     
    \begin{align*}
        &\|x(t) - \cen(t)\| < \rad(t) \Rightarrow \|x(t) - \cen(t)\| \uparrow \rad(t) \\
        \Rightarrow & \lim_{\|x(t) - \cen(t)\| \uparrow \rad(t)} \frac{d}{dt} \|x(t) - \cen(t)\| > \dot{\rad}(t) \\
        \Rightarrow &\lim_{\|x(t) - \cen(t)\| \uparrow \rad(t)} (x(t) - \cen(t))^\top (\dot{x}(t) - \dot{\cen}(t)) > \dot{\rad}(t)\|x(t)-\cen(t)\| \\
        \Rightarrow &\lim_{\|x(t) - \cen(t)\| \uparrow \rad(t)} (x(t) - \cen(t))^\top \dot{x}(t) >  \lim_{\|x(t) - \cen(t)\| \uparrow \rad(t)}(x(t) - \cen(t))^\top \dot{\cen}(t) + \dot{\rad}(t)\|x(t) - \cen(t)\|.
    \end{align*}
    Substituting the dynamics from Equation~\eqref{eqn:sysDyn_vel} with velocity input from \eqref{eqn:bdvelcon_ball},
    \begin{align*}
        &\lim_{\|x(t) - \cen(t)\| \uparrow \rad(t)} (x(t) - \cen(t))^\top v > \lim_{\|x(t) - \cen(t)\| \uparrow \rad(t)}(x(t) - \cen(t))^\top \dot{\cen}(t)+\dot{\rad}(t)\|x(t) - \cen(t)\|\\
        \Rightarrow &\lim_{\|x(t) - \cen(t)\| \uparrow \rad(t)} \!\!(x(t) \!-\! \cen(t))^\top \Bigg(-\frac{(x(t) \!-\! \cen(t))}{\|x(t) \!-\! \cen(t)\|} \|\bar{v}\|+e_v\Bigg) \!> \lim_{\|x(t) - \cen(t)\| \uparrow \rad(t)} \!\!(x(t) \!-\! \cen(t))^\top \dot{\cen}(t)+\dot{\rad}(t)\|x(t) \!-\! \cen(t)\| \\
        \Rightarrow &\lim_{\|x(t) - \cen(t)\| \uparrow \rad(t)} \|x(t) - \cen(t)\| (-\|\bar{v}\| + \|p_v\|) > \lim_{\|x(t) - \cen(t)\| \uparrow \rad(t)} \|x(t) - \cen(t)\|(\|\dot{\cen}(t)\| + \dot{\rad}(t)) \\
        \Rightarrow &\|\bar{v}\| < \overline{\Gamma} + \|p_v\|.
    \end{align*}
    However, this contradicts the feasibility constraint in Equation \eqref{eqn:feas1}. Hence, $\|x(t) - \cen(t)\| \nrightarrow \rad(t), \forall t \in [0,\tx)$, i.e., $x(t)$ never approaches the STT boundary over $t \in [0,\tx)$.
    Consequently, due to the continuity of $x(t)$, it can be concluded that there is no $\tx$ at which $x(t)$ crosses the STT boundary.
    
    Therefore, the reference velocity vector $v_r$ in~\eqref{eqn:bdvelcon_ball} constrains $x(t)$ within the STT, $\|x(t) - \cen(t)\| < \rad(t), \forall t \in \R^+_0.$

\textbf{Stage 2:}
Let $\tx$ be the first time instance when the velocity error $e_v(t)$, on the application of input $\tau(t)$ as in \eqref{eqn:bdcontrol}, violates \eqref{eqn:fun2}, i.e., $\exists i \in [1;n]$,
    $$e_{v,i}(\tx) \leq -\gamma_{v,i}(\tx) \text{ or } e_{v,i}(\tx) \geq \gamma_{v,i}(\tx).$$
Then, for all $(t,i) \in [0, \tx) \times [1;n],$
\begin{gather}\label{Eq:inqe_tv}
    -\gamma_{v,i}(t) < e_{v,i}(t) < \gamma_{v,i}(t).
\end{gather}
We will consider the following two cases for $t \in [0,\tx)$.

\textbf{Case I.} For some $i \in [1;n]$, $e_{v,i}(t)$ approaches the upper funnel constraint, i.e., $e_{v,i}(t) \rightarrow \gamma_{v,i}(t) \Rightarrow e_{v,i}(t) - \gamma_{v,i}(t) =: \overline{\delta}_{v,i} \rightarrow 0$. Following \eqref{Eq:inqe_tv}:
\begin{align*}
    &e_{v,i}(t) < \gamma_{v,i}(t) \Rightarrow \overline{\delta}_{v,i} \uparrow 0 \Rightarrow \lim_{\overline{\delta}_{v,i} \uparrow 0} \frac{d}{dt} \overline{\delta}_{v,i} > 0 \\
    \Rightarrow &\lim_{\overline{\delta}_{v,i} \uparrow 0} \dot{e}_{v,i}(t) > \lim_{\overline{\delta}_{v,i} \uparrow 0} \dot{\gamma}_{v,i}(t) > -\mu_{v,i}(p_{v,i}-q_{v,i}) \\
    \Rightarrow &\lim_{\overline{\delta}_{v,i} \uparrow 0} \dot{v}_{i}(t) > -\mu_{v,i}(p_{v,i}-q_{v,i}) + \dot{v}_{r,i}(t) > -\mu_{v,i}(p_{v,i}-q_{v,i}) - \overline{a}_{r,i}.
\end{align*}
Therefore, there exists $i \in [1;n]$, such that 
\begin{gather}\label{eqn:dv_b1}
    \lim_{\overline{\delta}_{v,i} \uparrow 0} \dot{v}_{i}(t) > -\mu_{v,i}(p_{v,i}-q_{v,i}) - \overline{a}_{r,i}.
\end{gather}
Since, $\lim_{\overline{\delta}_{v,i} \uparrow 0} \varepsilon_{v,i}(t) = 1$, we obtain $\lim_{\overline{\delta}_{v,i} \uparrow 0} \tau_{i}(t) = -\overline{\tau}_{i}$.
Using the dynamics~\eqref{eqn:sysDyn_acc} and feasibility condition \eqref{eqn:feas2}
\begin{align*}
    \lim_{\overline{\delta}_{v,i} \uparrow 0}  \dot{v}_{i}(t) \leq \overline{V}_{M,i} - \underline{m} \overline{\tau}_i + \underline{m}_i\overline{d}_i
    \leq -\mu_{v,i}(p_{v,i} - q_{v,i}) - \overline{a}_{r,i}
\end{align*}
which contradicts~\eqref{eqn:dv_b1}. 
Hence, $e_{v,i}(t) \nrightarrow \gamma_{v,i}(t), \forall (t,i) \in [0,\tx) \times [1;n]$, i.e., the velocity error $e_v(t)$ never approaches the upper funnel constraint over $t \in [0,\tx)$ in any dimension.

\textbf{Case II.} For some $i \in [1;n]$, $e_{v,i}(t)$ approaches the lower funnel constraint, i.e., $e_{v,i}(t) \rightarrow -\gamma_{v,i}(t) \Rightarrow e_{v,i}(t)+\gamma_{v,i}(t) =: \underline{\delta}_{v,i} \rightarrow 0$. Following \eqref{Eq:inqe_tv}:
\begin{align*}
    &e_{v,i}(t) > -\gamma_{v,i}(t) \Rightarrow \underline{\delta}_{v,i} \downarrow 0 \Rightarrow \lim_{\underline{\delta}_{v,i} \downarrow 0} \frac{d}{dt} \underline{\delta}_{v,i} < 0 \\
    \Rightarrow &\lim_{\underline{\delta}_{v,i} \downarrow 0} \dot{e}_{v,i}(t) < \lim_{\underline{\delta}_{v,i} \downarrow 0} -\dot{\gamma}_{v,i}(t) < \mu_{v,i}(p_{v,i}-q_{v,i}) \\
    \Rightarrow &\lim_{\underline{\delta}_{v,i} \uparrow 0} \dot{v}_{i}(t) < \mu_{v,i}(p_{v,i}-q_{v,i}) + \dot{v}_{r,i}(t)  < \mu_{v,i}(p_{v,i}-q_{v,i}) + \overline{a}_{r,i}.
\end{align*}
Therefore, there exists $i \in [1;n]$, such that
\begin{gather}\label{eqn:dv_b2}
    \lim_{\underline{\delta}_{v,i} \uparrow 0} \dot{v}_{i}(t) < \mu_{v,i}(p_{v,i}-q_{v,i}) + \overline{a}_{r,i}.
\end{gather}
Since, $\lim_{\underline{\delta}_{v,i} \downarrow 0} \varepsilon_{v,i}(t) = -1$, we obtain $\lim_{\underline{\delta}_{v,i} \downarrow 0} \tau_{i}(t) = \overline{\tau}_i.$ Using dynamics~\eqref{eqn:sysDyn_acc} and feasibility condition~\eqref{eqn:feas2}
\begin{align*}
    \lim_{\underline{\delta}_{v,i} \downarrow 0}  \dot{v}_{i}(t) \geq \underline{V}_{M,i} + \underline{m} \overline{\tau}_i - \underline{m}_i\overline{d}_i \geq \mu_{v,i}(p_{v,i} - q_{v,i}) + \overline{a}_{r,i}
\end{align*}
which contradicts~\eqref{eqn:dv_b2}. 
Hence, $e_{v,i}(t) \nrightarrow -\gamma_{v,i}(t), \forall (t,i) \in [0,\tx) \times [1;n]$, i.e., the velocity error $e_v(t)$ never approaches the lower funnel constraint over $t \in [0,\tx)$ in any dimension.

Thus, over $t \in [0, \tx])$, $e_{v,i}(t)$ never approaches the funnel constraints $-\gamma_{v,i}(t)$ and $\gamma_{v,i}(t)$ for all $i \in [1;n]$.
Consequently, due to the continuity of $e_v(t)$, it can be concluded that there is no $\tx$ at which $e_{v,i}(t)$ violates the funnel constraints $-\gamma_{v,i}(t)$ and $\gamma_{v,i}(t)$ for all $i \in [1;n]$.

Therefore, combining Stages~I and~II, the bounded control input $\tau(t)$ in \eqref{eqn:bdcontrol}, under feasibility conditions \eqref{eqn:feas1} and \eqref{eqn:feas2}, ensures that the system state $x(t)$ evolves within the STTs $\Gamma(t)$. This concludes the proof.
\end{proof}


\begin{figure*}[t]
    \centering
    \includegraphics[width=0.95\linewidth]{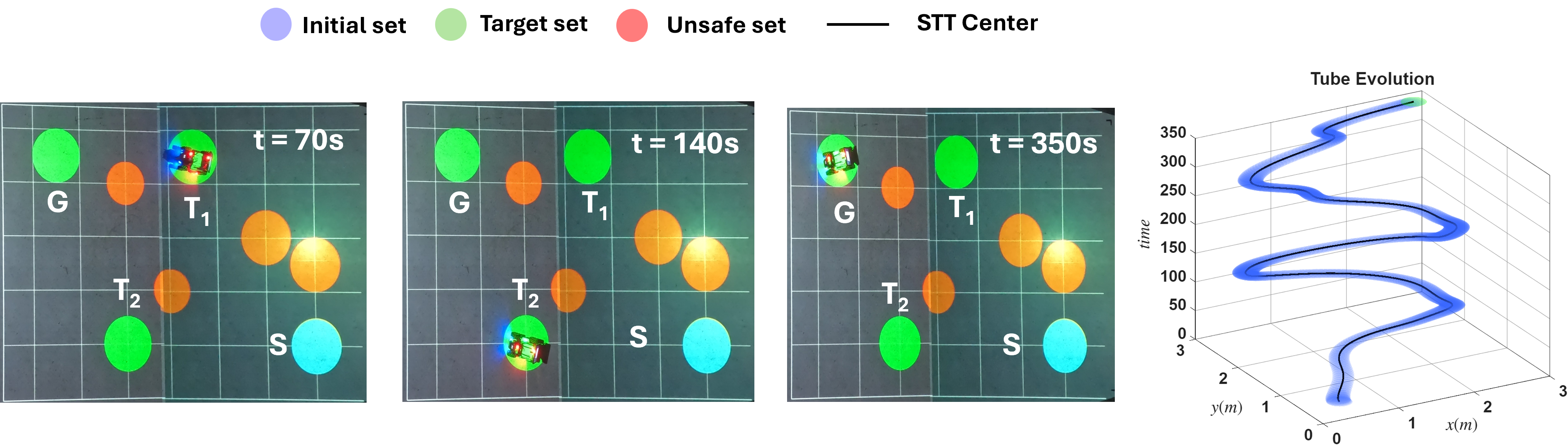}
    \caption{Mobile robot (Agile LIMO) positioning at different timestamps and the constructed STT. \href{https://drive.google.com/file/d/1pA__oWxuOKhNZ8FC_a4Puyd2x5GPKQMj/view?usp=sharing}{Video Link}}
    \label{fig:omni}
\end{figure*}

\section{Case Study}
To validate the effectiveness of the proposed approach, we present three different case studies: a 2D mobile robot, a 3D UAV, a $7$ DOF Manipulator, and a multi-agent setup. All experiments were conducted in PyTorch (Python 3.10) on a Windows machine with an Intel Core i7-14700 CPU, 32 GB RAM, and an NVIDIA GeForce RTX 3080 Ti GPU. All the videos can be found here \footnote{https://l1nk.dev/a4xbr9y}.

In case studies, each region $M$ is a hyperrectangle with center $\mathrm{center}(M)$ and size $d_M$, and is represented by the predicate $\tilde{M} := |x - \mathrm{center}(M)|_\infty < d_M$.

\subsection{Omnidirectional Robot}

For the first case study, we consider the omnidirectional robot whose dynamics is given by:
\begin{subequations}\label{eq:omni}
\begin{align}
\begin{bmatrix}
    \dot{x} & \dot{y} 
\end{bmatrix} &=  
\begin{bmatrix}
    v_x & v_y
\end{bmatrix} + d_1(t), \\
\begin{bmatrix}
    \dot{v}_x & \dot{v}_y
\end{bmatrix} &= 
\begin{bmatrix}
    a_x & a_y
\end{bmatrix} + d_2(t).
\end{align}
\end{subequations}

We consider the velocity and acceleration limits in each dimension are set to $0.15 \ m/s$ and $0.5 \ m/s^2$, respectively. We consider $\mathcal{L}_c = 0.08$ and $\mathcal{L}_r = 0.01$, satisfying condition \eqref{eqn:feas1}. Now, for the robot, we define the following STL task:
\begin{align*}
    \psi &= \square_{[0, 280]} \big((\Tilde{\So} \Rightarrow \Diamond_{[60, 80]}\Tilde{\T}_1) \land (\Tilde{\T}_1 \Rightarrow \Diamond_{[60, 80]}\Tilde{\T}_2) \land (\Tilde{\T}_2 \Rightarrow \Diamond_{[60, 80]}\Tilde{\T}_1) \big)  \land \Diamond_{[345, 355]} \square_{[0, 10]}\Tilde{\G} \land \square_{[0, 360]} \neg \Tilde{\Obs},
\end{align*}
with $\So = \B([0.5,0.5],0.15), \T_1 = \B([2.5, 1.5],0.15), \T_2 = \B([0.5,2.5], 0.15), \G = \B([2.75, 2.75],0.15)$ and $\Obs = \B([1.25, 0.5], 0.15) \cup \B([1.5, 1],0.25) \cup \B([1, 2],0.25) \cup \B([2, 2.5], 0.15)$.
This specification captures a surveillance task in which the omnibot, starting from the initial region $\So$, must repeatedly visit the targets $\T_1$ and $\T_2$ in sequence every $60-80$ seconds until $280$ seconds have elapsed. After completing this, the robot must reach the goal G between $345$ and $355$ seconds, while avoiding the obstacles $\Obs$ throughout the $360$-second mission. To test the robustness of our approach, we also introduce unknown but bounded disturbances. We first collect samples from the augmented state-space as described in \eqref{eq:sample} and apply Algorithm~\ref{algo:NN_training} to obtain the PINSTT that encapsulates the STL specification. Then, start the robot inside the tube and apply \eqref{eqn:bdvelcon_ball} and \eqref{eqn:bdcontrol} to the robot such that it remains within the tube and moves along it to satisfy the STL specification.

The offline computation time to obtain the STT through the training of PINN is 51.410 seconds, while the per step online control synthesis time is $3.17 \ \mu$sec.


We experimentally validated the approach using a hardware setup with an Agile LIMO robot. To test the robustness of our approach, we have performed experiments with nominal control and payload-variation control by attaching a 1~kg payload on top of the robot. The hardware video is available at \href{https://drive.google.com/file/d/1pA__oWxuOKhNZ8FC_a4Puyd2x5GPKQMj/view?usp=sharing}{Video Link} with snapshots of the robot with the learned STT is shown in Figure~\ref{fig:omni} while the control inputs are shown in Figure \ref{fig:limo_test}. The control effort occasionally spikes within the prescribed limits due to wheel slip near the tube boundary.

\begin{figure}[ht]
\centering
\includegraphics[width=0.75\linewidth, height=0.25\linewidth]{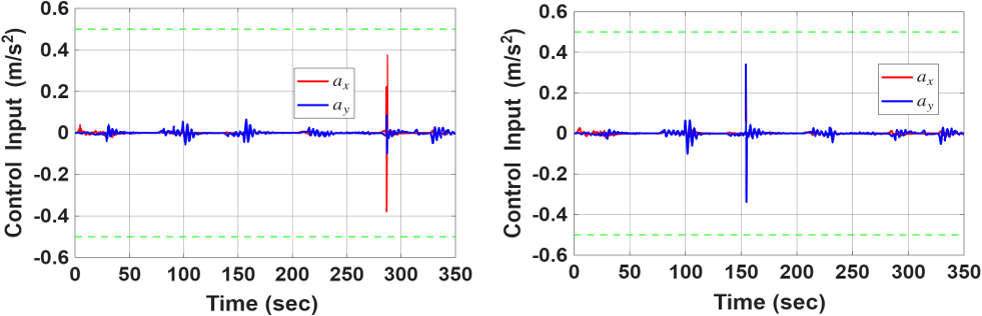}
\caption{Control inputs remain within bounds under nominal (left) and payload-variation (right) conditions.}
\label{fig:limo_test}
\end{figure}

\subsection{Quadrotor}

\begin{figure}[t]
    \centering
    \includegraphics[width=0.65\linewidth]{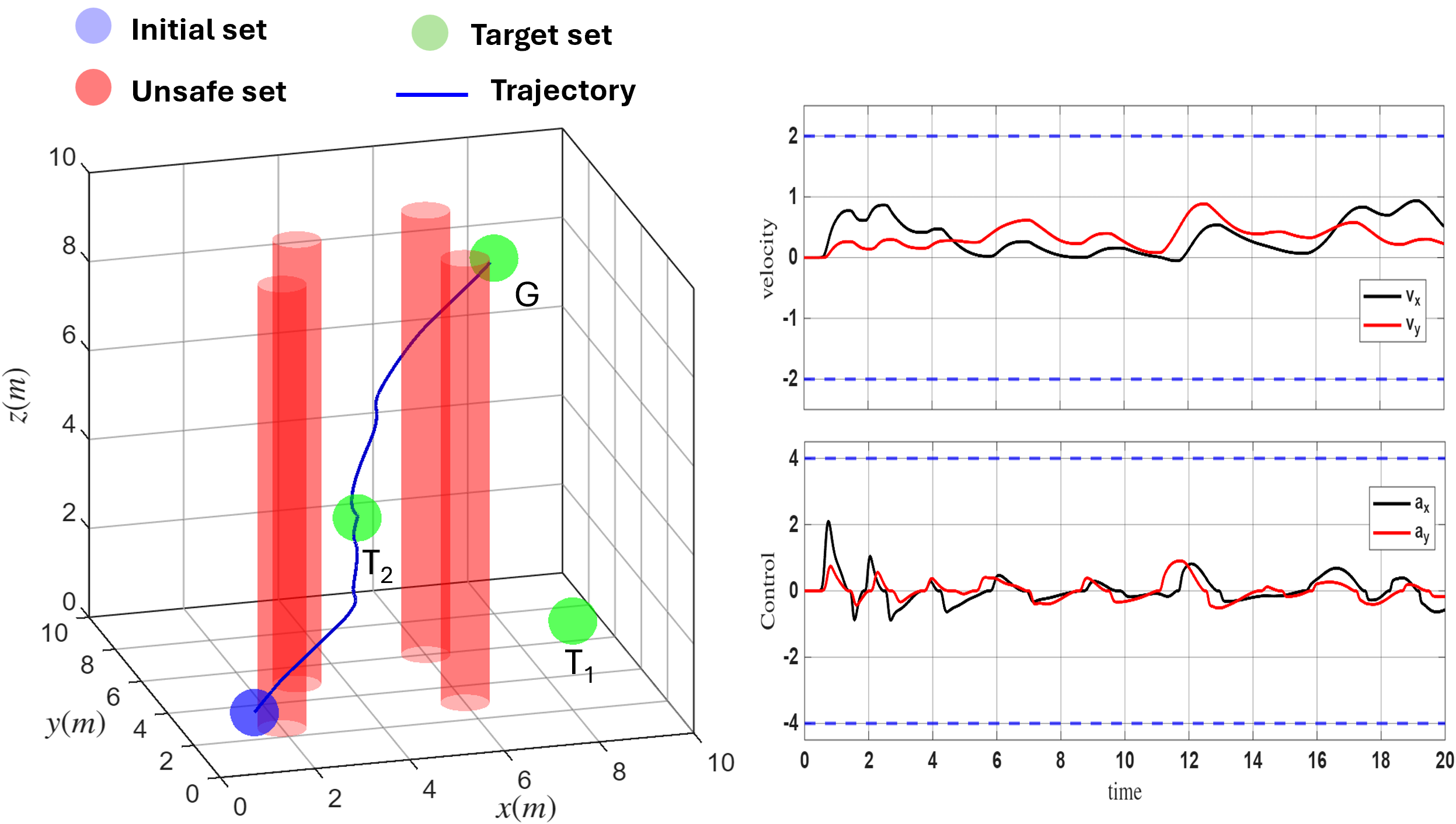}
    \caption{Quadrotor trajectory (left). Velocity and control inputs remain within the prescribed bounds (right). \href{https://drive.google.com/file/d/1kF4NjDsS4QaSV_7xyObfzDzOQzjDZ2H-/view?usp=sharing}{Video Link}}
    \label{fig:drone}
\end{figure}

We next evaluate the framework on a Quadrotor operating in a 3D environment with second-order dynamics \cite{APF_drone}:
\begin{subequations}\label{eq:drone}
\begin{align*}
\begin{bmatrix}
    \dot{x} & \dot{y} & \dot{z}
\end{bmatrix} &= 
\begin{bmatrix}
    v_x & v_y & v_z
\end{bmatrix} + d_1(t), \\
\begin{bmatrix}
    \dot{v}_x & \dot{v}_y & \dot{v}_z
\end{bmatrix} &= 
\begin{bmatrix}
    a_x & a_y & a_z
\end{bmatrix} + d_2(t).
\end{align*}
\end{subequations}

For the drone, the velocity and acceleration limits are set as $2 \ m/s$ and $4 \ m/s^2$ in each dimension respectively. Therefore, we have considered $\mathcal{L}_c = 1$ and $\mathcal{L}_r = 0.5$ and the funnel as $\gamma_v(t) = (0.4-0.002)e^{-2t} + 0.002$. So, the feasibility condition \eqref{eqn:feas1} has been satisfied as $\overline{v} = 2 > \bar{\Gamma} + p_v = 1+0.5+0.4 = 1.9$. We consider the transformation function $\Psi(x)$ to be $\tanh(5x^3)$. Therefore, with the choice of $\bar{a}_r = 2.5$, the feasibility condition \eqref{eqn:feas2} is satisfied with some bounded disturbance. Now STL specification is given as:
\begin{align*}
    \psi = \Tilde{\So} \Rightarrow \Diamond_{[0, 10]}(\Tilde{\T}_1 \lor \Tilde{\T}_2) \land \Diamond_{[19, 20]} \Tilde{\G} \land \square_{[0,  20]} \neg \Tilde{\Obs},
\end{align*}
where $\So = \B([1, 1, 1], 0.5), \T_1 = \B([8, 2, 2], 0.5), \T_2 = \B([4, 4, 4], 0.5), \G = \B([8, 8, 8], 0.5)$. The unsafe set $\Obs$ is a union of multiple static cylindrical obstacles centered at $[2,2.5], [6,3], [6,6], [3,5]$ with radius $0.5$. 
This STL formula specifies a sequential temporal task where the drone, starting from the initial region $\So$, must eventually reach either region $\T_1$ or $\T_2$ within 0-10 seconds, and subsequently proceed to the goal region $\G$ within $19-20$ seconds, while always avoiding the obstacles $\Obs$.
Figure \ref{fig:drone} shows the quadrotor trajectory satisfying the given STL specification while the control input and velocity are bounded within the specified bounds. The simulation video is available in the \href{https://drive.google.com/file/d/1kF4NjDsS4QaSV_7xyObfzDzOQzjDZ2H-/view?usp=sharing}{Video Link}. The offline PINSTT computation time is 5.284 seconds, while the per-step online control synthesis time is $3.41 \ \mu$sec. 

\subsection{7 DOF Franka Research 3 Manipulator}

\begin{figure}[ht]
\centering
\begin{subfigure}{0.45\textwidth}
    \centering
    \includegraphics[width=0.85\textwidth]{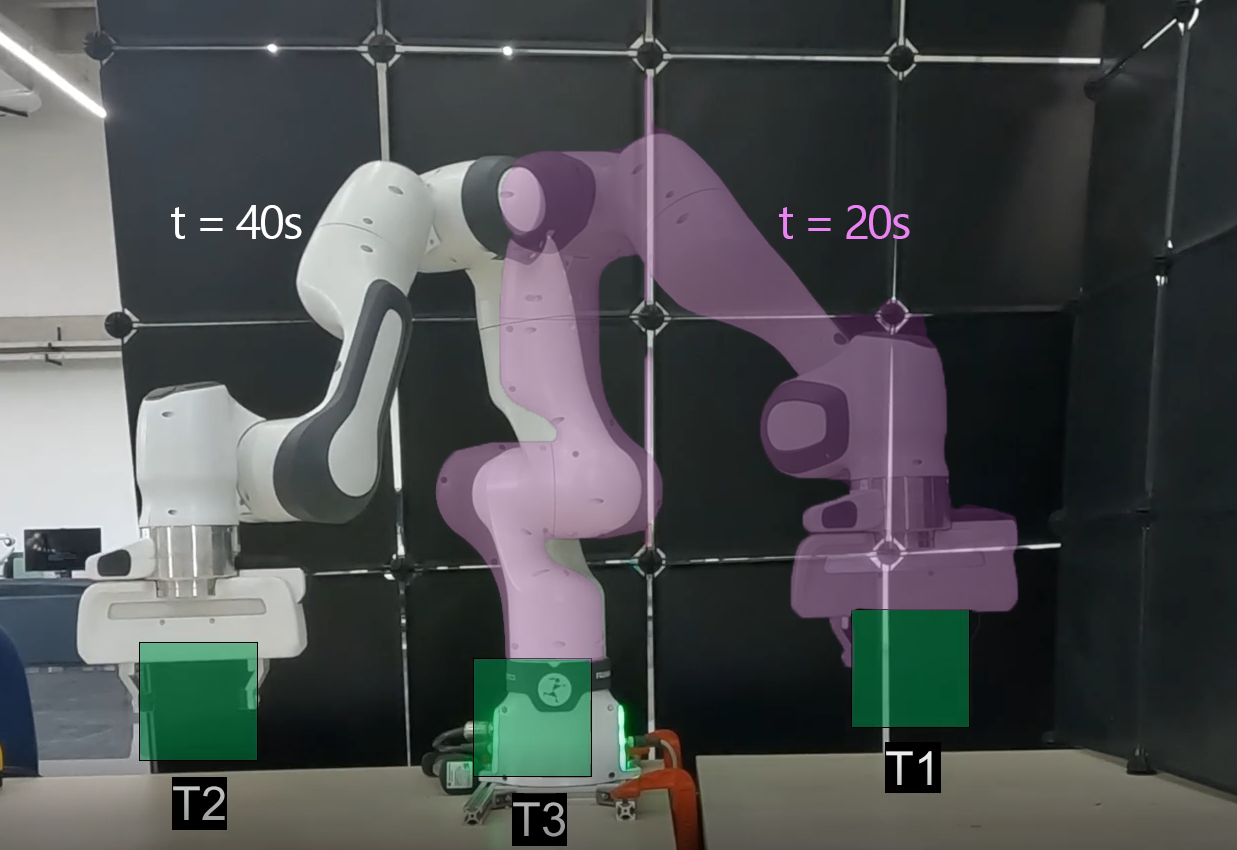}
    \caption{}
\end{subfigure}
\hfill
\begin{subfigure}{0.45\textwidth}
    \centering
    \includegraphics[width=0.95\textwidth, height=0.8\linewidth]{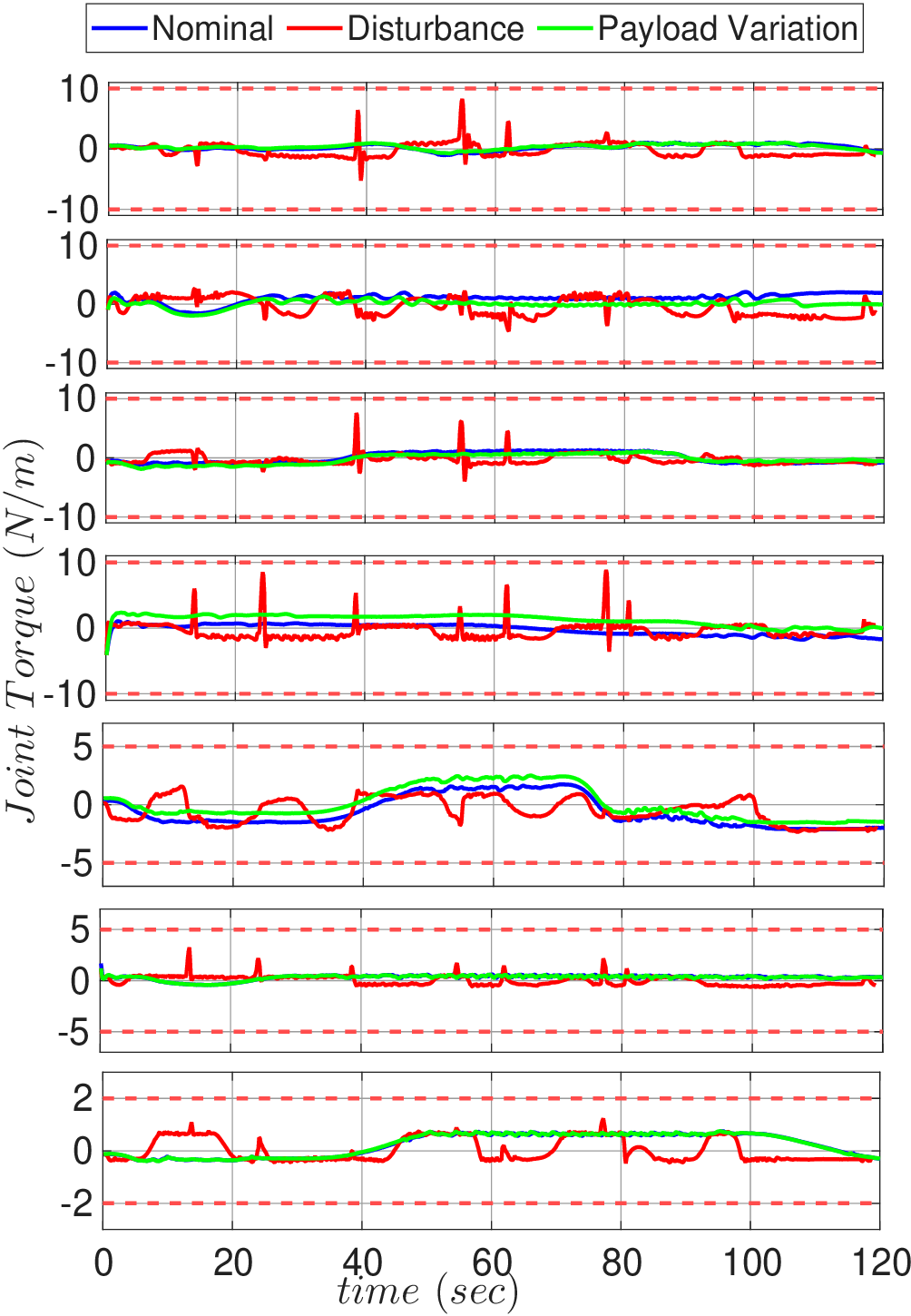}
    \caption{}
\end{subfigure}
\vspace{-0.2cm}
\caption{7-DOF Franka Research 3 Manipulator: (a) The end-effector reaches ${\T}_1, {\T}_2, {\T}_3$ satisfying the STL specification. (b) Joint torques remain bounded under nominal control, payload variation, and external disturbances. \href{https://drive.google.com/file/d/1Gg1KMfZ1AVjDMXuFT8gOVqyAn0ZwhXSC/view}{Video Link}}
\label{fig:franka_seq_reach}
\end{figure}

For the third case study, we use a 7-DOF Franka Emika Research 3 robot with joint limits $\overline{v} = 2.5 \ rad/s$ for all the joints and $\overline{\tau} = [10, 10, 10, 10, 5, 5, 2]^\top \ N/m$.
We define a sequential reachability task for the manipulator, defined by the following STL logic in its task-space:
\begin{align*}
    \psi &= \square_{[0, 120]}\big((\Tilde{\So} \Rightarrow \Diamond_{[18,23]} \Tilde{\T}_1 \land (\Tilde{\T}_1 \Rightarrow \Diamond_{[18,23]} \Tilde{\T}_2) \land (\Tilde{\T}_2 \Rightarrow \Diamond_{[18,23]} \Tilde{\T}_3) \land (\Tilde{\T}_3 \Rightarrow \Diamond_{[18,23]} \Tilde{\T}_1) \land (\Tilde{\T}_1 \Rightarrow \Diamond_{[18,23]} \Tilde{\So}) \big).
\end{align*}
To ensure the feasibility conditions, we have chosen $\mathcal{L}_c = 1.5$ and $\mathcal{L}_r = 0.5$. The end-effector positions are chosen in a way such that the joint configuration adheres to their respective bounds. We have tested our algorithm in three different categories: (a) nominal sequential reachability, (b) payload variation (attaching a water bottle to one of the links) to illustrate robustness to unknown dynamics, and (c) external disturbances. Figure~\ref{fig:franka_seq_reach} (b) shows that the control inputs (applied joint torques) for all the cases are bounded within the prescribed limits for each joint, while the experiment video can be found in the \href{https://drive.google.com/file/d/1Gg1KMfZ1AVjDMXuFT8gOVqyAn0ZwhXSC/view?usp=sharing}{Video Link}. The offline PINSTT computation takes roughly one minute, while the per-step online control synthesis time is $5.42\ \mu$sec. 


\begin{table*}[t]
\centering
\caption{Qualitative comparison of the proposed PINSTT framework with representative STL planning and control approaches.}
\label{tab:comp_qualitative}
\resizebox{0.88\textwidth}{!}{
\begin{tabular}{lcccccccc}
\hline
\textbf{Method} &
\textbf{\begin{tabular}[c]{@{}c@{}}Closed-form\\Controller\end{tabular}} &
\textbf{\begin{tabular}[c]{@{}c@{}}Formal STL\\Guarantee\end{tabular}} &
\textbf{\begin{tabular}[c]{@{}c@{}}Full STL\\Support\end{tabular}} &
\textbf{\begin{tabular}[c]{@{}c@{}}Unknown\\Dynamics\end{tabular}} &
\textbf{\begin{tabular}[c]{@{}c@{}}Bounded\\Disturbance\end{tabular}} &
\textbf{\begin{tabular}[c]{@{}c@{}}Input\\Constraints\end{tabular}} &
\textbf{\begin{tabular}[c]{@{}c@{}}Online\\Optimization\end{tabular}}
\\
\hline

RRT* \cite{linard2023real}
& -$^{1}$
& \xmark
& \cmark
& -$^{1}$
& -$^{1}$
& -$^{1}$
& -$^{1}$
\\

MPC \cite{raman2014model}
& \xmark
& \xmark
& \cmark
& \xmark
& \cmark
& \cmark
& \cmark
\\

CBF-based methods \cite{liu2023safe}
& \xmark
& \cmark
& \xmark
& \xmark
& \xmark
& \cmark
& \cmark
\\

PPC \cite{lindemann2017prescribed}
& \cmark
& \cmark
& \xmark
& \cmark
& \cmark
& \xmark
& \xmark
\\

MILP \cite{raman2015reactive}
& \xmark
& \xmark
& \cmark
& \xmark
& \xmark
& \xmark
& \cmark
\\

STT \cite{STT_STL}
& \cmark
& \cmark
& \cmark
& \cmark
& \cmark
& \xmark
& \xmark
\\

\textbf{PINSTT (Ours)}
& \cmark
& \cmark
& \cmark
& \cmark
& \cmark
& \cmark
& \xmark
\\

\hline
\end{tabular}
}

\vspace{1mm}
\footnotesize
\text{1.} Additional planning and control mechanisms are required to guarantee STL satisfaction, prescribed-time reachability, and robustness against bounded disturbances.

\end{table*}

\begin{figure*}[t]
    \centering
    \includegraphics[width=0.9\linewidth]{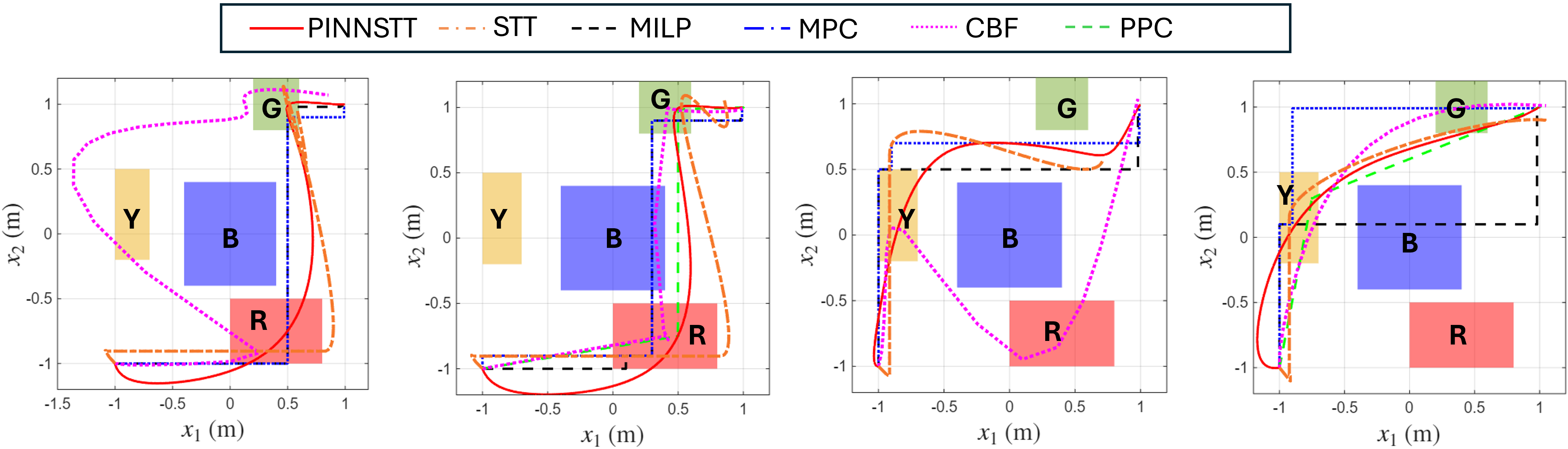}
    \caption{Comparison of the proposed PINSTT framework with existing STT, MILP, MPC, CBF, and PPC-based approaches for the four benchmark tasks. The tasks in the picture are $stlcg-1, stlfrag-1, stlcg-2, stlfrag-2$ respectively. }
    \label{fig:comparison}
\end{figure*}

\begin{table*}[t]
\centering
\caption{Computation Time (seconds)}
\label{tab:comparison}
\begin{tabular}{|l|l|l|l|l|l|l|}
\hline
Case      & \multicolumn{1}{c|}{PINSTT}                                   & \multicolumn{1}{c|}{STT \cite{STT_STL}}                                       & \multicolumn{1}{c|}{MILP \cite{raman2015reactive}} & \multicolumn{1}{c|}{MPC \cite{raman2014model}} & \multicolumn{1}{c|}{CBF \cite{liu2023safe}} & \multicolumn{1}{c|}{PPC \cite{lindemann2021funnel}} \\ \hline
stlcg-1   & 8.066+0.002 = \textbf{8.068} & 28.479+0.025 = 28.504 & 51.533                    & 44.532                   & 181.533                   & N/A                      \\ \hline
stlfrag-1 & 3.714+0.002 = \textbf{3.716} & 0.251+0.163 = 0.414 & 1.941                     & 1.936                    & 113.564                   & 0.076                    \\ \hline
stlcg-2   & 3.783+0.002 = \textbf{3.785} & 7.571+0.018 = 7.589  & 157.959                   & 159.288                  & 40.754                   & N/A                      \\ \hline
stlfrag-2 & 2.407+0.002 = \textbf{2.409} & 0.151+0.005 = 0.156 & 2.060                     & 2.424                    & 25.089                    & 0.090                    \\ \hline
\end{tabular}
\end{table*}


\section{Comparison}

\subsection{Qualitative Comparison}
To further highlight the strengths of the proposed framework, we qualitatively compare it with representative STL planning and control approaches. While sampling-based planners such as RRT* \cite{linard2023real} can generate collision-free paths, they require an additional tracking controller to execute the planned trajectory and do not provide formal guarantees. MPC-based methods \cite{raman2014model} can accommodate full STL specifications while explicitly handling input constraints and bounded disturbances; however, they require an accurate system model and rely on computationally intensive online optimization. Similarly, CBF-based approaches \cite{liu2023safe} provide formal safety guarantees and can incorporate actuator limits and bounded disturbances through optimization-based formulations, but they require known system dynamics and are generally limited to a fragment of STL specifications. PPC-based approaches \cite{lindemann2021funnel} admit closed-form controllers and can handle unknown dynamics and bounded disturbances, but are again applicable only to restricted STL fragments. Likewise, MILP-based approaches \cite{raman2015reactive} can express the full STL language, but require known system dynamics and computationally expensive online optimization. Our previous STT framework \cite{STT_STL} supports the full class of STL specifications for unknown systems under bounded disturbances through a closed-form controller; however, it does not explicitly account for actuator limits during tube synthesis or controller design. In contrast, the proposed PINSTT framework simultaneously supports the full class of STL specifications, unknown dynamics, bounded disturbances, and prescribed input constraints through a bounded closed-form controller, while avoiding online optimization. A qualitative comparison is summarized in Table~\ref{tab:comp_qualitative}.

\subsection{Quantitative Comparison}

We evaluate the proposed framework on two benchmark STL tasks, stlcg-1 and stlcg-2, introduced in \cite{sun2022multi}, together with their simplified STL fragment counterparts considered in \cite{lindemann2017prescribed}:
\begin{align*}
    &\text{stlcg-1:} (\Diamond_{[0, 15]}\square_{[0,5]}\Tilde{\mathbf{R}}) \land (\Diamond_{[0, 15]}\square_{[0,5]}\Tilde{\mathbf{G}}) \land (\square_{[0,15]}\neg \Tilde{\mathbf{B}}) \\
    &\text{stlcg-2:}(\Diamond_{[0, 15]}\square_{[0,5]}\Tilde{\mathbf{Y}}) \land (\square_{[0,15]}\neg \Tilde{\mathbf{G}}) \land (\square_{[0,15]}\neg \Tilde{\mathbf{B}}) \\
    &\text{stlfrag-1:} (\Diamond_{[0, 15]}\square_{[0,5]}\Tilde{\mathbf{R}}) \land (\Diamond_{[0, 15]}\square_{[0,5]}\Tilde{\mathbf{G}}) \\
    &\text{stlfrag-2:}(\Diamond_{[0, 15]}\square_{[0,5]}\Tilde{\mathbf{Y}}).    
\end{align*}
Figure~\ref{fig:comparison} illustrates the resulting trajectories obtained using a simulation time step of $0.01$ s.

\subsubsection{Computation Time comparison}

Table~\ref{tab:comparison} compares the computation time of the proposed method with existing approaches. For the STT-based methods, the first term denotes the offline tube synthesis time, while the second denotes the online control synthesis time. Although the previous STT approach \cite{STT_STL} is faster for simple STL fragments, its offline computation increases significantly with the complexity of the STL specification due to the underlying SMT solver. In contrast, the proposed PINN-based tube synthesis scales better for complex specifications. Furthermore, the proposed bounded transformation function results in a lower online control synthesis time than the earlier STT controller, which relied on an unbounded transformation function. Compared with MILP \cite{raman2015reactive}, MPC \cite{raman2014model}, and CBF-based methods \cite{liu2023safe}, the proposed framework also achieves significantly lower online computation by avoiding online optimization.


\begin{figure*}[t]
    \centering
    \includegraphics[width=0.8\linewidth,height=0.27\linewidth]{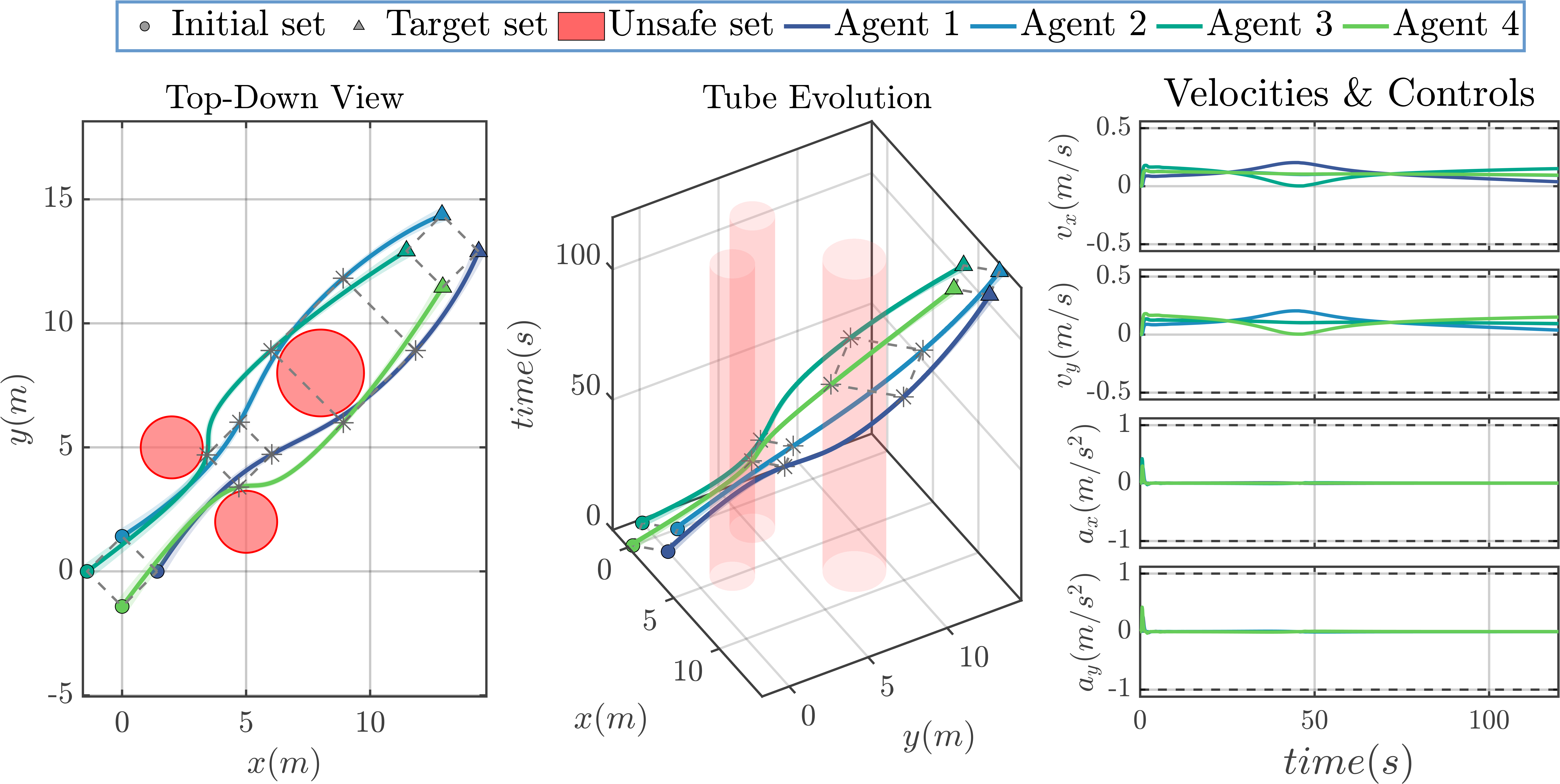}
    \caption{Multi-agent formation-preserving navigation.
    \textit{Left:} Top-down workspace projection showing circular obstacles (unsafe sets), initial/target configurations, and grey dotted snapshots highlighting the preservation of relative geometry and bearing throughout the mission. 
    \textit{Middle:} 3D space-time trajectories illustrating the continuous evolution of the formation lifted along the temporal axis. 
    \textit{Right:} Velocity and acceleration profiles demonstrating within given bounds for all agents. \href{https://drive.google.com/file/d/1SFuv0_H3HBsuI8AAtYDXU2pqJbePnGfo/view?usp=sharing}{Video Link}.}
    \label{fig:omni_formation}
\end{figure*}

\subsubsection{Input Comparison}

Table~\ref{tab:comparison_ip} compares the maximum control input for the benchmark tasks stlcg-1 and stlfrag-1. The proposed controller respects the prescribed input limit of $0.75$ throughout the task, both under nominal and perturbed conditions. In contrast, the previous STT controller \cite{STT_STL} requires substantially larger control inputs, especially under model perturbation, since it does not explicitly consider actuator limits. CBF- and MILP-based approaches also require larger control efforts under nominal conditions and are not evaluated under perturbations, since model inaccuracies can lead to violation of the STL specification.


\begin{table}[ht]
\centering
\caption{Maximum control input over complete time horizon comparison for STL specifications}
\label{tab:comparison_ip}
\resizebox{0.48\textwidth}{!}{
\begin{tabular}{|cc|cccc|}
\hline
\multicolumn{2}{|c|}{\multirow{2}{*}{Case Studies}}                                                              & \multicolumn{4}{c|}{Input Efforts}                                                              \\ \cline{3-6} 
\multicolumn{2}{|c|}{}                                                                                           & \multicolumn{1}{c|}{PINSTT} & \multicolumn{1}{c|}{STT \cite{STT_STL}}     & \multicolumn{1}{c|}{CBF \cite{liu2023safe}}    & MILP \cite{raman2015reactive} \\ \hline
\multicolumn{1}{|c|}{\multirow{2}{*}{stlcg-1}}   & \begin{tabular}[c]{@{}c@{}}without\\ disturbance\end{tabular} & \multicolumn{1}{c|}{0.75}   & \multicolumn{1}{c|}{1.2658}  & \multicolumn{1}{c|}{1.4461} &  3.200    \\ \cline{2-6} 
\multicolumn{1}{|c|}{}                           & \begin{tabular}[c]{@{}c@{}}with\\ disturbance\end{tabular}    & \multicolumn{1}{c|}{0.75}   & \multicolumn{1}{c|}{7.1901}  & \multicolumn{1}{c|}{-}      & -    \\ \hline
\multicolumn{1}{|c|}{\multirow{2}{*}{stlfrag-1}} & \begin{tabular}[c]{@{}c@{}}without\\ disturbance\end{tabular} & \multicolumn{1}{c|}{0.75}   & \multicolumn{1}{c|}{1.2937}  & \multicolumn{1}{c|}{0.7533}   &   3.600   \\ \cline{2-6} 
\multicolumn{1}{|c|}{}                           & \begin{tabular}[c]{@{}c@{}}with\\ disturbance\end{tabular}    & \multicolumn{1}{c|}{0.75}   & \multicolumn{1}{c|}{10.6582} & \multicolumn{1}{c|}{-}      & -    \\ \hline
\end{tabular}
}
\end{table}



\section{Extension to Multi-Agent Systems}
To extend the approach for a multi-agent scenario, we define the robustness metric of the complete cascaded formulation of individual agents and apply a similar strategy to generate the PINSTT for all the agents by defining a robustness metric for the overall cascaded system. Then we apply a similar control strategy for individual agents to satisfy the specification. We validate the effectiveness of the proposed approach in multi-agent scenarios using a team of omnidirectional robots, whose dynamics is given in \eqref{eq:omni}. We perform in simulation, a group of agents navigating in a bounded 2D workspace with static circular obstacles while maintaining a prescribed formation. 

We evaluate a formation-preserving navigation task featuring four agents initialized in a diamond configuration. The objective is to translate the formation diagonally across the workspace while navigating obstacles. This scenario presents significant coupling challenges: agents must simultaneously avoid environmental collisions, maintain safe inter-agent separation, and preserve the global formation structure. Because deviations by any single agent propagate through the collective, coordinated motion is essential. 

The task is formally defined using Signal Temporal Logic as a conjunction of local and global specifications:
\begin{align*}
\psi = \; 
\underbrace{
\left( \wedge_{i=1}^{4} \Tilde{\So}_i \right)
}_{\text{Initial condition}}\; 
& \land\;
\underbrace{
\square_{[0,120]} \left( \wedge_{i=1}^{4} \neg \Tilde{\Obs}_i \right)
\;\land\; 
\wedge_{i=1}^{4} \left( \Diamond_{[0,120]} \Tilde{\textbf{G}}_i \right)
}_{\text{Local specifications}} \land \; 
\underbrace{
\left(\square_{[0,120]} \Tilde{\textbf{F}}_{\text{form}}\right)
\;\land\; 
\square_{[0,120]} \left( \wedge_{i=1}^{4} \neg \Tilde{\textbf{X}}_i \right)
}_{\text{Global specification}}
\end{align*}
where $\Tilde{\Obs}_i$ represents unsafe regions, $\Tilde{\mathbf{F}}_{\text{form}}$ encodes leader-relative distance and bearing constraints, and $\Tilde{\mathbf{G}}_i$ denotes target goal regions for the $i^{\text{th}}$ agent, $i=\{1,2,3,4\}$. Also, $\So_1 = \B([\sqrt{2},0],0.3)$, $\So_2 = \B([0,\sqrt{2}],0.3)$, $\So_3 = \B([-\sqrt{2},0],0.3)$, $\So_4 = \B([0,-\sqrt{2}],0.3)$, $\mathbf{G}_1 = \B([13+\sqrt{2},13],0.3)$, $\mathbf{G}_2 = \B([13,13+\sqrt{2}],0.3)$, $\mathcal{G}_3 = \B([13-\sqrt{2},13],0.3)$, $\mathbf{G}_4 = \B([13,13-\sqrt{2}],0.3)$, $\Obs = \B([5,2],1.25) \cup \B([2,5],1.25) \cup \B([8,8],1.75)$, representing static obstacles, and $\Tilde{\mathbf{X}}_i = \bigcup_{\substack{j=1 \\ j \neq i}}^{4} \B([x_j,y_j],0.1),$ representing agents acting as dynamic obstacles to agent $i$. The desired formation is specified relative to the leader agent through nominal inter-agent distances proportional to $[2, 2\sqrt{2}, 2]$ and relative bearings $\B\left(\frac{3\pi}{4}, 0.1\right)$, $\B\left(\pi, 0.1\right)$, $\B\left(-\frac{3\pi}{4}, 0.1\right)$. 

The velocity and acceleration limits in each dimension are set to $0.5 \ \text{m/s}$ and $1.0 \ \text{m/s}^2$, respectively. To satisfy the feasibility conditions in \eqref{eqn:feas1}, we select Lipschitz constants for the tube parameters as $\mathcal{L}_c = 0.4$ and $\mathcal{L}_r = 0.05$. Formation integrity is maintained via distance and bearing constraints relative to a designated leader. A formation consistency objective penalizes geometric deviations while permitting uniform scaling, allowing the formation to adapt to constrained environments without compromising its structural identity.

The offline training of the PINN requires $59.5$ seconds while the per-step online control synthesis time is $3.17 \ \mu\text{s}$. Figure~\ref{fig:omni_formation} illustrates the resulting trajectories in 3D space-time and top-down projections. The learned STTs generate smooth, feasible trajectories that successfully navigate obstacles while maintaining formation throughout the mission horizon. Velocity and acceleration profiles remain strictly within prescribed bounds. The simulation video can be found in the \href{https://drive.google.com/file/d/1SFuv0_H3HBsuI8AAtYDXU2pqJbePnGfo/view?usp=sharing}{Video Link}. 


\section{Conclusion and Future Work} 

This paper presented a physics informed neural spatiotemporal tube (PINSTT) framework for satisfying Signal Temporal Logic (STL) specifications for unknown Euler--Lagrange systems subject to input constraints. The proposed approach parameterizes the time-varying center and radius of the tube using a physics-informed neural network, where the STL robustness metric is directly incorporated into the training objective. To provide formal guarantees over the continuous time horizon, Lipschitz continuity of the tube parameters is enforced through physics-informed loss functions using automatic differentiation, and a verification condition is derived to certify the validity of the synthesized tube. Subsequently, a closed-form, approximation-free, model-free control law is developed to ensure that the system trajectory remains within the synthesized tube while respecting prescribed input bounds. The proposed framework combines several desirable properties within a single architecture: support for the full class of STL specifications, applicability to unknown Euler-Lagrange systems, bounded closed-form control inputs, and microsecond-level online control synthesis. The effectiveness of the approach has been demonstrated through both simulation and hardware experiments involving single-agent and multi-agent scenarios.

Future work will focus on extending the framework to stochastic systems and probabilistic settings, incorporating dynamic environments, and developing distributed tube synthesis and control strategies for large-scale multi-agent systems.

\bibliographystyle{plain}
\bibliography{sources} 

@article{ruder2016,
  title={An overview of gradient descent optimization algorithms},
  author={Ruder, Sebastian},
  journal={arXiv preprint arXiv:1609.04747},
  year={2016}
}

@article{das2025spatiotemporal,
  author={Das, Ratnangshu and Basu, Ahan and Jagtap, Pushpak},
  journal={IEEE Transactions on Automatic Control}, 
  title={Spatiotemporal Tubes for Temporal Reach-Avoid-Stay Tasks in Unknown Systems}, 
  year={2026},
  volume={71},
  number={1},
  pages={512-519}
}

@ARTICLE{APF_drone,
  author={Pan, Zhenhua and Zhang, Chengxi and Xia, Yuanqing and Xiong, Hao and Shao, Xiaodong},
  journal={IEEE Transactions on Circuits and Systems II: Express Briefs}, 
  title={An Improved Artificial Potential Field Method for Path Planning and Formation Control of the Multi-{UAV} Systems}, 
  year={2022},
  volume={69},
  number={3},
  pages={1129-1133}}

@ARTICLE{CBF,
  author={Ames, Aaron D. and Xu, Xiangru and Grizzle, Jessy W. and Tabuada, Paulo},
  journal={IEEE Transactions on Automatic Control}, 
  title={Control Barrier Function Based Quadratic Programs for Safety Critical Systems}, 
  year={2017},
  volume={62},
  number={8},
  pages={3861-3876}}

@inproceedings{ames2019control,
  title={Control barrier functions: Theory and applications},
  author={Ames, Aaron D and Coogan, Samuel and Egerstedt, Magnus and Notomista, Gennaro and Sreenath, Koushil and Tabuada, Paulo},
  booktitle={IEEE 18th European control conference (ECC)},
  pages={3420--3431},
  year={2019}
}

@article{STT_STL,
  author={Das, Ratnangshu and Choudhury, Subhodeep and Jagtap, Pushpak},
  journal={IEEE Control Systems Letters}, 
  title={Approximation-Free Control for Signal Temporal Logic Specifications Using Spatiotemporal Tubes}, 
  year={2025},
  volume={9},
  number={},
  pages={1562-1567}
}

@inproceedings{maler2004monitoring,
  title={Monitoring temporal properties of continuous signals},
  author={Maler, Oded and Nickovic, Dejan},
  booktitle={International symposium on formal techniques in real-time and fault-tolerant systems},
  pages={152--166},
  year={2004},
  organization={Springer}
}

@inproceedings{donze2010robust,
  title={Robust satisfaction of temporal logic over real-valued signals},
  author={Donz{\'e}, Alexandre and Maler, Oded},
  booktitle={International conference on formal modeling and analysis of timed systems},
  pages={92--106},
  year={2010},
  organization={Springer}
}

@article{liu2021recurrent,
  title={Recurrent neural network controllers for signal temporal logic specifications subject to safety constraints},
  author={Liu, Wenliang and Mehdipour, Noushin and Belta, Calin},
  journal={IEEE Control Systems Letters},
  volume={6},
  pages={91--96},
  year={2021},
  publisher={IEEE}
}

@article{saxena2023funnel,
  title={Funnel-based reward shaping for signal temporal logic tasks in reinforcement learning},
  author={Saxena, Naman and Gorantla, Sandeep and Jagtap, Pushpak},
  journal={IEEE Robotics and Automation Letters},
  volume={9},
  number={2},
  pages={1373--1379},
  year={2023},
  publisher={IEEE}
}

@article{sun2022multi,
  title={Multi-agent motion planning from signal temporal logic specifications},
  author={Sun, Dawei and Chen, Jingkai and Mitra, Sayan and Fan, Chuchu},
  journal={IEEE Robotics and Automation Letters},
  volume={7},
  number={2},
  pages={3451--3458},
  year={2022},
  publisher={IEEE}
}

@article{lindemann2018control,
  title={Control barrier functions for signal temporal logic tasks},
  author={Lindemann, Lars and Dimarogonas, Dimos V},
  journal={IEEE Control Systems Letters},
  volume={3},
  number={1},
  pages={96--101},
  year={2018},
  publisher={IEEE}
}

@article{akella2023lipschitz,
  title={Lipschitz continuity of signal temporal logic robustness measures: Synthesizing control barrier functions from one expert demonstration},
  author={Akella, Prithvi and Badithela, Apurva and Murray, Richard M and Ames, Aaron D},
  journal={arXiv preprint arXiv:2304.03849},
  year={2023}
}

@article{das2025control,
  title={Control Barrier Functions for the Full Class of Signal Temporal Logic Tasks using Spatiotemporal Tubes},
  author={Das, Ratnangshu and Choudhury, Subhodeep and Jagtap, Pushpak},
  journal={arXiv preprint arXiv:2510.19595},
  year={2025}
}

@inproceedings{puranic2021learning,
  title={Learning from demonstrations using signal temporal logic},
  author={Puranic, Aniruddh and Deshmukh, Jyotirmoy and Nikolaidis, Stefanos},
  booktitle={Conference on Robot Learning},
  pages={2228--2242},
  year={2021},
  organization={PMLR}
}

@inproceedings{juvvi2025signal,
  title={Signal Temporal Logic Compliant Co-design of Planning and Control},
  author={Juvvi, Manas Sashank and Kurne, Tushar Dilip and Vaishnavi, J and Kolathaya, Shishir and Jagtap, Pushpak},
  booktitle={IEEE/RSJ International Conference on Intelligent Robots and Systems (IROS)},
  pages={19204--19209},
  year={2025}
}

@inproceedings{raman2014model,
  title={Model predictive control with signal temporal logic specifications},
  author={Raman, Vasumathi and Donz{\'e}, Alexandre and Maasoumy, Mehdi and Murray, Richard M and Sangiovanni-Vincentelli, Alberto and Seshia, Sanjit A},
  booktitle={53rd IEEE Conference on Decision and Control},
  pages={81--87},
  year={2014}
}

@inproceedings{raman2015reactive,
  title={Reactive synthesis from signal temporal logic specifications},
  author={Raman, Vasumathi and Donz{\'e}, Alexandre and Sadigh, Dorsa and Murray, Richard M and Seshia, Sanjit A},
  booktitle={Proceedings of the 18th international conference on hybrid systems: Computation and control},
  pages={239--248},
  year={2015}
}

@article{leung2023backpropagation,
  title={Backpropagation through signal temporal logic specifications: Infusing logical structure into gradient-based methods},
  author={Leung, Karen and Ar{\'e}chiga, Nikos and Pavone, Marco},
  journal={The International Journal of Robotics Research},
  volume={42},
  number={6},
  pages={356--370},
  year={2023},
  publisher={SAGE Publications Sage UK: London, England}
}

@book{tabuada2009verification,
  title={Verification and control of hybrid systems: a symbolic approach},
  author={Tabuada, Paulo},
  year={2009},
  publisher={Springer Science \& Business Media}
}

@article{lindemann2021funnel,
  title={Funnel control for fully actuated systems under a fragment of signal temporal logic specifications},
  author={Lindemann, Lars and Dimarogonas, Dimos V},
  journal={Nonlinear Analysis: Hybrid Systems},
  volume={39},
  pages={100973},
  year={2021},
  publisher={Elsevier}
}

@inproceedings{liu2022compositional,
  title={Compositional synthesis of signal temporal logic tasks via assume-guarantee contracts},
  author={Liu, Siyuan and Saoud, Adnane and Jagtap, Pushpak and Dimarogonas, Dimos V and Zamani, Majid},
  booktitle={IEEE 61st Conference on Decision and Control (CDC)},
  pages={2184--2189},
  year={2022}
}

@inproceedings{liu2023learning,
  title={Learning robust and correct controllers from signal temporal logic specifications using barriernet},
  author={Liu, Wenliang and Xiao, Wei and Belta, Calin},
  booktitle={62nd IEEE Conference on Decision and Control (CDC)},
  pages={7049--7054},
  year={2023}
}

@inproceedings{lindemann2017prescribed,
  title={Prescribed performance control for signal temporal logic specifications},
  author={Lindemann, Lars and Verginis, Christos K and Dimarogonas, Dimos V},
  booktitle={IEEE 56th Annual Conference on Decision and Control (CDC)},
  pages={2997--3002},
  year={2017}
}

@inproceedings{liu2023safe,
  title={Safe Model-based Control from Signal Temporal Logic Specifications Using Recurrent Neural Networks},
  author={Liu, Wenliang and Nishioka, Mirai and Belta, Calin},
  booktitle={IEEE International Conference on Robotics and Automation (ICRA)},
  pages={12416--12422},
  year={2023}
}

@article{sharma2022accelerated,
  title={Accelerated training of physics-informed neural networks (pinns) using meshless discretizations},
  author={Sharma, Ramansh and Shankar, Varun},
  journal={Advances in neural information processing systems},
  volume={35},
  pages={1034--1046},
  year={2022}
}

@article{das2026prescribedperformancecontrolunknown,
      title={Prescribed Performance Control of Unknown {E}uler-{L}agrange Systems Under Input Constraints}, 
      author={Ratnangshu Das and Pushpak Jagtap},
      journal={arXiv preprint arXiv:2507.01426},
      year={2025}
}

@book{ELbook,
  title={Euler-Lagrange systems},
  author={Ortega, Romeo and Loria, Antonio and Nicklasson, Per Johan and Sira-Ramirez, Hebertt and Ortega, Romeo and Lor{\'\i}a, Antonio and Nicklasson, Per Johan and Sira-Ram{\'\i}rez, Hebertt},
  year={1998},
  publisher={Springer}
}

@book{ELbook2,
author = {Siciliano, Bruno and Sciavicco, Lorenzo and Villani, Luigi and Oriolo, Giuseppe},
title = {Robotics: Modelling, Planning and Control},
year = {2010},
isbn = {1849966346},
publisher = {Springer Publishing Company, Incorporated}
}

@article{ELbounds,
  title={Robot control by using only joint position measurements},
  author={Nicosia, S and Tomei, P},
  journal={IEEE Transactions on Automatic control},
  volume={35},
  number={9},
  pages={1058--1061},
  year={1990},
  publisher={IEEE}
}

@article{raissi2019physics,
  title={Physics-informed neural networks: A deep learning framework for solving forward and inverse problems involving nonlinear partial differential equations},
  author={Raissi, Maziar and Perdikaris, Paris and Karniadakis, George E},
  journal={Journal of Computational physics},
  volume={378},
  pages={686--707},
  year={2019},
  publisher={Elsevier}
}

@inproceedings{linard2023real,
  title={Real-time {RRT}* with signal temporal logic preferences},
  author={Linard, Alexis and Torre, Ilaria and Bartoli, Ermanno and Sleat, Alex and Leite, Iolanda and Tumova, Jana},
  booktitle={IEEE/RSJ International Conference on Intelligent Robots and Systems (IROS)},
  pages={8621--8627},
  year={2023}
  }

\end{document}